\documentclass[english]{elsarticle}
\usepackage[T1]{fontenc}
\usepackage[latin9]{inputenc}
\usepackage{float}
\usepackage{textcomp}
\usepackage{mathrsfs}
\usepackage{mathtools}
\usepackage{amsmath}
\usepackage{amsthm}
\usepackage{amssymb}
\usepackage{stmaryrd}
\usepackage{esint}

\makeatletter

\newcommand*\LyXThinSpace{\,\hspace{0pt}}
\newcommand{\noun}[1]{\textsc{#1}}
\providecommand{\tabularnewline}{\\}
\floatstyle{ruled}
\newfloat{algorithm}{tbp}{loa}
\providecommand{\algorithmname}{Algorithm}
\floatname{algorithm}{\protect\algorithmname}

\theoremstyle{plain}
\newtheorem{thm}{\protect\theoremname}
\theoremstyle{definition}
\newtheorem{defn}[thm]{\protect\definitionname}
\theoremstyle{plain}
\newtheorem{lem}[thm]{\protect\lemmaname}
\theoremstyle{plain}
\newtheorem{prop}[thm]{\protect\propositionname}
\theoremstyle{plain}
\newtheorem{cor}[thm]{\protect\corollaryname}
\theoremstyle{remark}
\newtheorem{rem}[thm]{\protect\remarkname}
\theoremstyle{definition}
\newtheorem{example}[thm]{\protect\examplename}

\makeatother

\usepackage{babel}
\providecommand{\corollaryname}{Corollary}
\providecommand{\definitionname}{Definition}
\providecommand{\examplename}{Example}
\providecommand{\lemmaname}{Lemma}
\providecommand{\propositionname}{Proposition}
\providecommand{\remarkname}{Remark}
\providecommand{\theoremname}{Theorem}

\begin{document}

\begin{frontmatter}{}

\title{Stochastic Global Optimization Algorithms: A Systematic Formal Approach}

\author[1]{Jonatan Gomez}

\fntext[]{Department of Computer Science and Engineering, Universidad Nacional
de Colombia, Bogotá, Colombia}

\ead{jgomezpe@unal.edu.co}
\begin{abstract}
As we know, some global optimization problems cannot be solved using
analytic methods, so numeric/algorithmic approaches are used to find
near to the optimal solutions for them. A stochastic global optimization
algorithm (\noun{SGoal}) is an iterative algorithm that generates
a new population (a set of candidate solutions) from a previous population
using stochastic operations. Although some research works have formalized
\noun{SGoal}s using Markov kernels, such formalization is not general
and sometimes is blurred. In this paper, we propose a comprehensive
and systematic formal approach for studying \noun{SGoal}s. First,
we present the required theory of probability ($\sigma$-algebras,
measurable functions, kernel, markov chain, products, convergence
and so on) and prove that some algorithmic functions like swapping
and projection can be represented by kernels. Then, we introduce the
notion of join-kernel as a way of characterizing the combination of
stochastic methods. Next, we define the optimization space, a formal
structure (a set with a $\sigma$-algebra that contains strict $\epsilon$-optimal
states) for studying \noun{SGoal}s, and we develop kernels, like sort
and permutation, on such structure. Finally, we present some popular
\noun{SGoal}s in terms of the developed theory, we introduce sufficient
conditions for convergence of a \noun{SGoal}, and we prove convergence
of some popular \noun{SGoal}s.\end{abstract}

\end{frontmatter}{}

\section{\label{sec:Stochastic-Global-Optimization}Stochastic Global Optimization}

A global optimization problem is formulated in terms of finding a
point $x$ in a subset $\varOmega\subseteq\varPhi$ where a certain
function $f\vcentcolon\varPhi\rightarrow\mathbb{R}$, attains is \textbf{best/optimal
}value (minimum or maximum) \cite{liberti2008introduction}. In the
optimization field, $\varOmega$, $\varPhi$, and $f$ are called
the feasible region, the solution space, and the objective function,
respectively. The optimal value for the objective function (denoted
as $f^{*}\in\mathbb{R}$) is suppose to exist and it is unique ($\mathbb{R}$
is a total order). In this paper, the global optimization problem
will be considered as the minimization problem described by equation
\ref{eq:GlobalMin}.

\begin{equation}
min\left(f:\varPhi\rightarrow\mathbb{R}\right)=x\in\varOmega\subseteq\varPhi\mid\left(\forall y\in\varOmega\right)\left(f\left(x\right)\leq f\left(y\right)\right)\label{eq:GlobalMin}
\end{equation}

Since a global optimization problem cannot be solved, in general,
using analytic methods, numeric methods are largely applied in this
task \cite{Rangaiah:2010:SGO:1875054,opac-b1132570}. Some numeric
methods are deterministic, like Cutting plane techniques \cite{doi:10.1137/0108053},
and Branch and Bound \cite{Morrison:2016:BA:2899521.2899589} approaches
while others are stochastic like Hill Climbing and Simulated Annealing.
Many stochastic methods, like Evolutionary Algorithms and Differential
Evolution, are based on heuristics and metaheuristics \cite{Holland75,DeJong75,Eiben99,fleetwood2004introduction}.
In this paper, we will concentrate on Stochastic Global Optimization
Methods (algorithms). A Stochastic Global Optimization ALgorithm (\noun{SGoal})
is an iterative algorithm that generates a new candidate set of solutions
(called population) from a given population using a stochastic operation,
see Algorithm \ref{Alg:SGoal}.

\begin{algorithm}[htbp]
\noun{SGoal}($n$)

1.~~$t_{0}=0$

2.~~$P$ = \noun{InitPop}($n$) 

3.~~\textbf{while} \noun{\textlnot End}($P_{t}$ , $t$) \textbf{do}

4.~~~~$P_{t+1}$ = \noun{NextPop}( $P_{t}$ )

5.~~~~$t=t+1$

6.~~\textbf{return} \noun{best}($P_{t}$)

\caption{\label{Alg:SGoal}Stochastic Global Optimization Algorithm.}
\end{algorithm}

In Algorithm \ref{Alg:SGoal}, $n$ is the number of individuals in
the population\textit{ }\textit{\emph{(population's size)}}\textit{}\footnote{Although, parameters that control the stochastic process, like the
size of the population, can be adapted (adjusted) during the execution
of a \noun{SGoal}, in this paper we just consider any \noun{SGoal}
with fixed control parameters (including population's size).}, \textit{$P_{t}\in\varOmega^{n}$} is the population at iteration
\textit{$t\geq0$},\noun{ $\mbox{\textsc{InitPop}:}\mathbb{N}\rightarrow\varOmega^{n}$}
is a function that generates the initial population\textit{ }\textit{\emph{(according
to some distribution)}}\textit{,}\textit{\noun{ $\mbox{\textsc{NextPop}:}\varOmega^{n}\rightarrow\varOmega^{n}$}}
is a stochastic method that generates the next population from the
current one (the stochastic search), \textit{\noun{$\mbox{\textsc{End}:}\varOmega^{n}\times\mathbb{N}\rightarrow Bool$}}
is a predicate that defines when the \noun{SGoal($n$)} process is
stopped and \noun{$\mbox{\textsc{Best}:}\varOmega^{n}\rightarrow\varOmega$}
is a function that obtains the best candidate solution (individual)
in the population according to the optimization problem under consideration,
see equation \ref{eq:Best}.

\noun{
\begin{equation}
\mbox{\textsc{Best}}\left(x\right)=x_{i}\mid\forall_{k=1}^{n}f\left(x_{i}\right)\leq f\left(x_{k}\right)\land f\left(x_{i}\right)<\forall_{k=1}^{i-1}f\left(x_{k}\right)\label{eq:Best}
\end{equation}
} 

Although there are several different \noun{SGoal} models, such models
mainly vary on the definition of the \noun{NextPop }function. Sections
\ref{sub:Hill-Climbing} to \ref{sub:Differential-Evolution-(DE)}
present three popular \noun{SGoal}s reported in the literature.

\subsection{\label{sub:Hill-Climbing}Hill Climbing (HC)}

The hill climbing algorithm (\textbf{HC}), see Algorithm \ref{Alg:Hill-Climbing},
is a \noun{SGoal} that uses a single individual as population ($n=1$),
generates a new individual from it (using the stochastic method \noun{Variate:$\varOmega\rightarrow\varOmega$}),
and maintains the best individual among them (line 2). Notice that
HC allows to introduce neutral mutations\footnote{A neutral mutation is a variation in the individual that does not
change the value of the objective function \cite{Kimura83}. } if the greather or equal operator ($\geq$) is used in line 2. In
order to maintain more than one individual in the population, the
HC algorithm can be parallelized using Algorithm \ref{Alg:Parallel-Hill-Climbing}.

\begin{algorithm}[htbp]
\noun{NextPop}$_{\textsc{HC}}$($\left\{ x\right\} $)

1.~~$x'=$\noun{ Variate}($x$)

2.~~\textbf{if} $f\left(x'\right)\left\{ >,\geq\right\} f\left(x\right)$
\textbf{then} \textbf{$x'=x$} 

3.~~\textbf{return} $\left\{ x'\right\} $

\caption{\label{Alg:Hill-Climbing}Hill Climbing Algorithm - \noun{NextPop}
Method.}
\end{algorithm}

\begin{algorithm}[htbp]
\noun{NextPop}$_{\textsc{PHC}}$($P$)

1.~~$\left\{ Q_{i}\right\} =$ \noun{NextPop}$_{\mbox{HC}}$($\left\{ P_{i}\right\} $)
for all $i=1,2,\ldots,\left|P\right|$

2.~~\textbf{return} $Q$

\caption{\label{Alg:Parallel-Hill-Climbing}Parallel Hill Climbing Algorithm
(PHC) - \noun{NextPop} Method.}
\end{algorithm}

\subsection{\label{sub:Genetic-Algorithms-(Ga)}Genetic Algorithms (\noun{Ga})}

Genetic algorithms (\noun{Ga})s are optimization techniques based
on the principles of natural evolution \cite{Holland75}. Although
there are several different versions of \noun{Ga}s, such as Generational
Genetic (\textbf{\noun{GGa}}) algorithms and Steady State Genetic
(\textbf{\noun{SSGa}}) algorithms, in general, all G\noun{a}s have
the same structure. Major differences between them are in the encoding
scheme, in the evolution mechanism, and in the replacement mechanism.
Algorithms \ref{Alg:GGA} and \ref{Alg:SSGA} present the \noun{GGa}
and \noun{SSGa}, respectively. There, \noun{PickParents}:$\varOmega^{n}\rightarrow\mathbb{N}^{2}$
picks two individuals (indices) as parents, \noun{XOver:}$\varOmega^{2}\rightarrow\varOmega^{2}$
combines both of them and produces two new individuals, \noun{Mutate}:$\varOmega^{2}\rightarrow\varOmega^{2}$
produces two individuals (offspring) that are mutations of such two
new individuals, \noun{Best$_{2}$}:$\varOmega^{4}\rightarrow\varOmega^{2}$
picks the best two individuals between parents and offspring, and
\noun{Bernoulli}($r$) generates a \emph{true} value following a Bernoulli
distribution with probability $CR$.

\begin{algorithm}[htbp]
\noun{NextPop}$_{\textsc{GGa}}$($P$)

1.~~\textbf{for} $i=1$ \textbf{to} $\frac{n}{2}$

2.~~~~$\left\{ a,b\right\} =$\noun{PickParents}($P$)

3.~~~~\textbf{if} \noun{Bernoulli}($CR$)\noun{ }\textbf{then}
$\left\{ Q_{2i-1},Q_{2i}\right\} =$\noun{ Mutate(XOver}($P_{a}$,
$P{}_{b}$)) 

4.~~~~\textbf{else }$\left\{ Q_{2i-1},Q_{2i}\right\} =\left\{ P_{a},P_{b}\right\} $

5.~~\textbf{return} $Q$

\caption{\label{Alg:GGA}Generational Genetic Algorithm (\noun{GGa}) - \noun{NextPop}
Method.}
\end{algorithm}

\begin{algorithm}[htbp]
\noun{NextPop}$_{\textsc{SSGa}}$($P$)

1.~~$\left\{ a,b\right\} =$\noun{ PickParents}($P$)

2.~~$Q_{k}=P_{k}$ for all $k=1,2,\ldots\left|P\right|$, $k\neq a,b$

3.~~\textbf{if} \noun{Bernoulli}($CR$)\noun{ }\textbf{then} $\left\{ c_{1},c_{2}\right\} =$\noun{
Mutate(XOver}($P_{a}$, $P_{b}$))

4.~~\textbf{else }$\left\{ c_{1},c_{2}\right\} =$\noun{ Mutate(}$P_{a}$,
$P_{b}$)

5.~~\textbf{$\left\{ Q_{a},Q_{b}\right\} =$}\noun{ Best$_{2}$$\left(c_{1},c_{2},P_{a},P_{b}\right)$}

6.~~\textbf{return} $Q$

\caption{\label{Alg:SSGA}Steady State Genetic Algorithm (\noun{SSGa}) - \noun{NextPop}
Method.}
\end{algorithm}

\subsection{\label{sub:Differential-Evolution-(DE)}Differential Evolution (\noun{DE})}

Differential Evolution (\noun{DE}) algorithm is an optimization technique,
for linear spaces, based on the idea of using vector differences for
perturbing a candidate solution, see Algorithm \ref{Alg:DE}. Here,
$\varOmega$ is a $d$-dimensional linear search space, \noun{PickDifParents:$\mathbb{N}\times\mathbb{N}\rightarrow\mathbb{N}^{3}$
}gets three individuals (indices $a,$ $b$, and $c$) that are different
from each other and different from the individual under consideration
($i$), $0\leq CR\leq1$ is a crossover rate, and $0\leq F\leq2$
is the difference weight.

\begin{algorithm}[htbp]
\noun{NextInd}$_{\textsc{DE}}$($P$, $i$)

1.~~$\left\{ a,b,c\right\} =$\noun{ PickDifParents}($\left|P\right|$,
$i$) 

2.~~$R\sim\mathbb{N}\left[1,d\right]$

3.~~\textbf{for $k=1$ to $d$} 

5.~~\textbf{~~if} \noun{Bernoulli}($CR$)\noun{ }or $k=R$\noun{
}\textbf{then} $q_{k}=P_{a,k}+F*\left(P_{b,k}-P_{c,k}\right)$

6.~~~~\textbf{else} $q_{k}=P_{i,k}$

7.~~\textbf{return} $q$

~

\noun{NextPop}$_{\textsc{DE}}$($P$)

1.~~\textbf{for} $i=1$ \textbf{to} $n$

2.~~~~$Q_{i}=\mbox{\textsc{NextInd}}{}_{\textsc{DE}}\left(P,i\right)$

3.~~\textbf{return} $Q$

\caption{\label{Alg:DE}Differential Evolution Algorithm (\noun{DE}) - \noun{NextPop}
Method.}
\end{algorithm}

\section{Measure and Probability Theory }

In this section, we introduce the basic measure and probability theory
concepts that are required for a formal treatment of \noun{SGoal}s.
First, we will concentrate on the concept of family of sets, required
for formalizing concepts like event and random observation (subsections
\ref{sub:Family-of-Sets} and \ref{sub:sigma-algebras}). Next, we
will cover the concepts of measurable, measure, and probability measure
functions (subsection \ref{sub:Functions}) that are required for
defining and studying the notion of Kernel (subsection \ref{sub:Kernel}),
notion that will be used as formal characterization of stochastic
methods used by \noun{SGoal}s. Then, we will present the concept of
Markov chains (subsection \ref{sub:Markov-Chains}), concept that
is used for a formal treatment of \noun{SGoal}s. Finally, we introduce
two concepts of random sequence convergence (subsection \ref{sub:Convergence})
for studying the convergence properties of a \noun{SGoal}.

\subsection{\label{sub:Family-of-Sets}Family of Sets}

Probability theory starts by defining the space of elementary events
(a nonempty set $\varOmega$) and the system of observable events
(a family of subsets of $\Omega$). In the case of a formal treatment
of a \noun{SGoal}, the space of elementary events is the set of possible
populations while the system of observable events is defined by any
subset of populations that can be generated, starting from a single
population, through the set of stochastic methods used by the \noun{SGoal}.
In the rest of this paper, let $\varOmega\neq\textrm{Ø}$ be a non-empty
set (if no other assumption is considered).
\begin{defn}
(\textbf{Power Set}) Let $\varOmega$ be a set, the power set of $\varOmega$,
denoted as $2^{\varOmega}$, is the family of all subsets of $\varOmega$,
i.e. $2^{\varOmega}=\left\{ A\mid A\subseteq\varOmega\right\} $.
\end{defn}
Clearly, the system of observable events is a subset $\mathcal{A}$
of $2^{\varOmega}$ that satisfies some properties. Here, we introduce
families of sets with some properties that are required for that purpose.
Then, we stablish a relation between two of them.
\begin{defn}
\label{def:families}Let $\mathcal{A}\subseteq2^{\varOmega}$ be a
family of subsets of $\varOmega$. 
\begin{enumerate}
\item[\noun{$\mbox{\textsc{df}}$}]  (\textbf{disjoint family}) $\mathcal{A}$ is a disjoint family if
$A\bigcap B=\emptyset$ for any pair of $A\neq B\in\mathcal{A}$.
\item[\noun{$\mbox{\textsc{cf}}$}]  (\textbf{countable family}) $\mathcal{A}$ is a countable family
if $\mathcal{A}=\left\{ A_{i}\right\} _{i\in I}$ for some countable
set $I$.
\item[\noun{$\mbox{\textsc{cdf}}$}]  (\textbf{countable disjoint family}) $\mathcal{A}$ is a countable
disjoint family if $\mathcal{A}$ is \noun{$\mbox{\textsc{cf}}$}
and .
\item[\noun{$\overline{\mbox{\textsc{c}}}$}]  (\textbf{close under complements}) $\mathcal{A}$ is close under
complements if $A\in\mathcal{A}$ then $A^{c}\equiv\varOmega\setminus A\in\mathcal{A}$. 
\item[\noun{$\overline{\mbox{\textsc{pd}}}$}]  (\textbf{close under proper differences}) $\mathcal{A}$ is close
under proper differences if $A,B\in\mathcal{A}$ and $A\subset B$
then $B\setminus A\in\mathcal{A}$.
\item[\noun{$\overline{\mbox{\textsc{cdu}}}$}]  (\textbf{close under countable disjoint unions}) $\mathcal{A}$
is close under countable disjoint unions if $\bigcup_{i\in I}A_{i}\in\mathcal{A}$
for all $\left\{ A_{i}\in\mathcal{A}\right\} _{i\in I}$ \noun{$\mbox{\textsc{cdf}}$}. 
\item[\noun{$\overline{\mbox{\textsc{cu}}}$}]  (\textbf{close under countable unions}) $\mathcal{A}$ is close
under countable disjoint unions if $\bigcup_{i\in I}A_{i}\in\mathcal{A}$
for all $\left\{ A_{i}\in\mathcal{A}\right\} _{i\in I}$ \noun{$\mbox{\textsc{cf}}$}.
\item[\noun{$\overline{\mbox{\textsc{ci}}}$}]  (\textbf{close under countable intersections}) $\mathcal{A}$ is
close under countable intersections if $\bigcap_{i\in I}A_{i}\in\mathcal{A}$
for all $\left\{ A_{i}\in\mathcal{A}\right\} _{i\in I}$ \noun{$\mbox{\textsc{cf}}$}.
\item[\textbf{$\pi$}]  (\textbf{$\pi$-system}) $\mathcal{A}$ is a \textbf{$\pi$}-system
if it is close under finite intersections, i.e. if $A,B\in\mathcal{A}$
then $A\bigcap B\in\mathcal{A}$.
\item[\noun{$\lambda$}]  (\textbf{$\lambda$-system}) $\mathcal{A}$ is called\textbf{ $\lambda$}-system
iff ($\lambda.1$) $\textrm{Ø}\in\mathcal{A}$, ($\lambda.2$) $\mathcal{A}$
is \noun{$\overline{\mbox{\textsc{pd}}}$} and ($\lambda.3$) $\mathcal{A}$
is \noun{$\overline{\mbox{\textsc{cdu}}}$.}
\end{enumerate}
\end{defn}
\begin{lem}
\label{lem:lambda-cupd}Let $\mathcal{A}\subseteq2^{\varOmega}$ 
\begin{enumerate}
\item (\noun{$\lambda\rightarrow\overline{\mbox{\textsc{pd}}}$}) If $\mathcal{A}$
is $\lambda$-system then $\mathcal{A}$ is \noun{$\overline{\mbox{\textsc{pd}}}$}.
\item (\noun{$\overline{\mbox{\textsc{c}}}\rightarrow\left(\overline{\mbox{\textsc{cu}}}\leftrightarrow\overline{\mbox{\textsc{ci}}}\right)$})
If $\mathcal{A}$ is \noun{$\overline{\mbox{\textsc{c}}}$} then $\mathcal{A}$
is \noun{$\overline{\mbox{\textsc{cu}}}$ }iff $\mathcal{A}$ is \noun{$\overline{\mbox{\textsc{ci}}}$.}
\end{enumerate}
\end{lem}
\begin{proof}
\textbf{{[}1. }\noun{$\lambda\rightarrow\overline{\mbox{\textsc{pd}}}$}\textbf{{]}}
If $A,B\in\mathcal{A}$ then $B^{c}\in\mathcal{A}$ ($\lambda.2$:$\mathcal{A}$
is \noun{$\overline{\mbox{\textsc{c}}}$}). Clearly, $A\bigcap B^{c}=\textrm{Ø}$
($A\subset B$) and $A\bigcup B^{c}\in\mathcal{A}$ ($\lambda.3$:$\mathcal{A}$
is \noun{$\overline{\mbox{\textsc{cdu}}}$}). So, $\left(A\bigcup B^{c}\right)^{c}\in\mathcal{A}$
($\lambda.2$:$\mathcal{A}$ is \noun{$\overline{\mbox{\textsc{c}}}$}),
i.e., $A^{c}\bigcap B\in\mathcal{A}$ (Morgan's law). Therefore, $B\setminus A\in\mathcal{A}$
(def. proper difference). \textbf{{[}2. }\noun{$\overline{\mbox{\textsc{c}}}\rightarrow\left(\overline{\mbox{\textsc{cu}}}\leftrightarrow\overline{\mbox{\textsc{ci}}}\right)$}\textbf{{]}}
If $\left\{ A_{i}\right\} _{i\in I}$ is \noun{$\mbox{\textsc{cf}}$}
in \noun{$\mathcal{A}$, }then $\left\{ A_{i}^{c}\right\} _{i\in I}$
is \noun{$\mbox{\textsc{cf}}$} in $\mathcal{A}$ ($\mathcal{A}$
is \noun{$\overline{\mbox{\textsc{c}}}$})\noun{. }Now,\noun{ }if\noun{
}$\mathcal{A}$ is \noun{$\overline{\mbox{\textsc{cu}}}$} then \noun{${\displaystyle {\textstyle \bigcap}_{i\in I}}A_{i}=\left({\displaystyle {\textstyle \bigcup}_{i\in I}A_{i}^{c}}\right)^{c}\in\mathcal{A}$
(}Morgan's law and\noun{$\mathcal{A}$ }is \noun{$\overline{\mbox{\textsc{c}}}$}),
so $\mathcal{A}$ is \noun{$\overline{\mbox{\textsc{ci}}}$}. Finally,
if\noun{ }$\mathcal{A}$ is \noun{$\overline{\mbox{\textsc{ci}}}$}
then \noun{${\displaystyle {\textstyle \bigcup}_{i\in I}}A_{i}=\left({\displaystyle {\textstyle \bigcap}_{i\in I}A_{i}^{c}}\right)^{c}\in\mathcal{A}$
(}Morgan's law and\noun{ $\mathcal{A}$ }is \noun{$\overline{\mbox{\textsc{c}}}$}),
so $\mathcal{A}$ is \noun{$\overline{\mbox{\textsc{cu}}}$.}
\end{proof}

\subsection{\label{sub:sigma-algebras}$\sigma$-algebras}

Although each family of sets, in definition \ref{def:families}, is
very interesting on its own, none of them allows by itself to define,
in a consistent manner, a notion of probability. As we will see, $\sigma$-algebras
play this role in a natural way.
\begin{defn}
(\textbf{$\sigma$-algebra}) A family of sets $\Sigma\subseteq2^{\varOmega}$
is called a \textbf{$\sigma$}-algebra over $\varOmega$, iff ($\sigma.1$)
$\varOmega\in\Sigma$, ($\sigma.2$) $\Sigma$ is \noun{$\overline{\mbox{\textsc{c}}}$},
and ($\sigma.3$) $\Sigma$ is \noun{$\overline{\mbox{\textsc{cu}}}$.}

Now, we can stablish some relations between $\sigma$-algebras and
some of the previously defined families of sets. These relations are
very useful when dealing with notions like measure, measurable, and
kernel.\end{defn}
\begin{lem}
\label{lem:sigma-intersections}Let $\Sigma$ be a $\sigma$-algebra
over $\varOmega$.
\begin{enumerate}
\item $\textrm{Ø}\in\Sigma$
\item $\Sigma$ is \noun{$\overline{\mbox{\textsc{ci}}}$}.
\item $\Sigma$ is a $\lambda$-system.
\item $\Sigma$ is \noun{$\overline{\mbox{\textsc{pd}}}$}.
\end{enumerate}
\end{lem}
\begin{proof}
\textbf{{[}1{]}} $\varOmega\in\Sigma$ ($\sigma.1$) then $\textrm{Ø}\in\Sigma$
($\sigma.2$:$\Sigma$ is \noun{$\overline{\mbox{\textsc{c}}}$}).
\textbf{{[}2{]}} Follows from $\sigma.2$, $\sigma.3$ and lemma \ref{lem:lambda-cupd}.
\textbf{{[}3{]}}\emph{ }$\lambda.1$ follows from (1), $\lambda.2$
and $\lambda.3$ follow from $\sigma.2$ and $\sigma.3$, respectively.
\textbf{{[}4{]}} Follows from (3) and lemma \ref{lem:lambda-cupd}.\end{proof}
\begin{prop}
Let $\varOmega$ be a set 
\begin{enumerate}
\item $2^{\varOmega}$ is a $\sigma$-algebra
\item If $\left\{ \Sigma_{i}\right\} _{i\in I}$ is a family of $\sigma$-algebras
over $\varOmega$ then $\bigcap_{i\in I}\Sigma_{i}$ is a $\sigma$-algebra
over $\varOmega$. 
\item If $\mathcal{A}\subseteq2^{\varOmega}$ is an arbitrary family of
subsets of $\varOmega$ then the minimum $\sigma$-algebra generated
by $\mathcal{A}$ is $\sigma\left(\mathcal{A}\right)=\bigcap\left\{ \Sigma\mid\mathcal{A}\subseteq\Sigma\mbox{ and }\Sigma\mbox{ is }\sigma\mbox{-algebra}\right\} $. 
\end{enumerate}
\end{prop}
\begin{proof}
\textbf{{[}1{]}} Obvious. \textbf{{[}2{]}} $\varOmega\in\Sigma_{i}$
for all $i\in I$ ($\sigma.1$) then $\varOmega\in\bigcap_{i\in I}\Sigma_{i}$
(def. $\bigcap$). If $A\in\bigcap_{i\in I}\Sigma_{i}$ then $A\in\Sigma_{i}$
for all $i\in I$, then $A^{c}\in\Sigma_{i}$ for all $i\in I$ ($\sigma.2:\Sigma_{i}$
is \noun{$\overline{\mbox{\textsc{c}}}$}), therefore $A^{c}\in\bigcap_{i\in I}\Sigma_{i}$
(def. $\bigcap$). If $\left\{ A_{j}\in\bigcap_{i\in I}\Sigma_{i}\right\} _{j\in J}$
is \noun{$\textsc{cf}$} then $A_{j}\in\Sigma_{i}$ for all $j\in J$
and $i\in I$, then $\left\{ A_{j}\in\Sigma_{i}\right\} _{j\in J}$
is \noun{$\textsc{cf}$} in $\Sigma_{i}$ for all $i\in I$, therefore
$\bigcup_{j\in J}A_{j}\in\Sigma_{i}$ for all $i\in I$ ($\sigma.3$:
$\Sigma_{i}$ is\noun{ $\overline{\textsc{cu}}$}). So, $\bigcup_{j\in J}A_{j}\in\bigcap_{i\in I}\Sigma_{i}$
(def. $\bigcap$).\emph{ }\textbf{{[}3{]}} Follows from (1) and (2).\end{proof}
\begin{thm}
\label{thm:Dynkin---theorem.}(\textbf{Dynkin $\pi$-$\lambda$ theorem})
Let $\mathcal{A}$ be a $\lambda$-system and let $\mathcal{E}\subseteq\mathcal{A}$
be a $\pi$-system then $\sigma\left(\mathcal{E}\right)\subseteq\mathcal{A}$.\end{thm}
\begin{proof}
A proof of this theorem can be found on page 6 of Kenkle's book \cite{Kenkle14}
(Theorem 1.19).
\end{proof}
Now, notions of measure and probability measure are defined on the
real numbers ($\mathbb{R}$), usually equipped with the Euclidean
distance, so we need to define an appropiated $\sigma$-algebra on
it. Such appropiated $\sigma$-algebra can be defined as a special
case of a $\sigma$-algebra for topological spaces. 
\begin{defn}
\label{def:(Borel--algebra)}(\textbf{Borel $\sigma$-algebra}) Let
$\left(\varOmega,\tau\right)$ be a topological space. The $\sigma$-algebra
$\mathcal{B}\left(\varOmega\right)\equiv\mathcal{B}\left(\varOmega,\tau\right)\equiv\sigma\left(\tau\right)$
is called the Borel $\sigma$-algebra on $\varOmega$ and every $A\in\mathcal{B}\left(\varOmega,\tau\right)$
is called Borel (measurable) set.\end{defn}
\begin{prop}
\emph{\label{prop:half-open-borel-real}If $\mathcal{B}\left(\mathbb{R}\right)$
is the Borel $\sigma$-algebra where $\mathbb{R}$ is equipped with
the Euclidean distance, then $\mathcal{B}\left(\mathbb{R}\right)=\sigma\left(\mathcal{E}_{7}\right)$
with $\mathcal{E}_{7}=\left\{ \left(\alpha,\beta\right]\mid\alpha,\beta\in\mathbb{Q},\alpha<\beta\right\} $.}\end{prop}
\begin{proof}
A proof of this proposition can be found on page 9 of Kenkle's book
\cite{Kenkle14} (Theorem 1.23).
\end{proof}
Now, we are ready to define the basic mathematical structure used
by probability theory.
\begin{defn}
(\textbf{measurable space}) If $\Sigma$ is a $\sigma$-algebra over
a set $\varOmega$ then the pair $\left(\varOmega,\Sigma\right)$
is called a measurable space. Sets in $\Sigma$ are called measurable
sets on $\varOmega$.
\end{defn}

\subsection{\label{sub:Functions}Functions}

Having the playground defined (space of elementary events and observable
events), probability theory defines operations over them (functions).
Such functions will allow us to characterize stochastic methods used
by a \noun{SGoal}\emph{. }First, we introduce the concepts of set
function and inverse function, that are used when working on $\sigma$-algebras.
\begin{defn}
(\textbf{set functions}) Let $f\colon\varOmega_{1}\rightarrow\varOmega_{2}$
be a function,
\begin{enumerate}
\item The power set function of $f$ is defined as

\[
\begin{array}{rccl}
f\colon & 2^{\varOmega_{1}} & \longrightarrow & 2^{\varOmega_{2}}\\
 & A & \longmapsto & \left\{ f\left(x\right)\mid\forall\left(x\in A\right)\right\} 
\end{array}
\]

\item The inverse function of $f$ is defined as

\[
\begin{array}{rccl}
f^{-1}\colon & 2^{\varOmega_{2}} & \longrightarrow & 2^{\varOmega_{1}}\\
 & B & \longmapsto & \left\{ x\in\varOmega_{1}\mid\left(\exists y\in B\right)\left(y=f\left(x\right)\right)\right\} 
\end{array}
\]

\end{enumerate}
\end{defn}
Next, we study the measurable functions, structure-preserving maps
(homomorphisms between measurable spaces). A measurable function guarantees
that observable events are obtained by applying the function to observable
events. For \noun{SGoal}s, a measurable function (stochastic methods)
basically means that any generated subset of populations must be obtained
by applying the stochastic methods to some generated subset of populations.
\begin{defn}
(\textbf{measurable function}) Let $\left(\varOmega_{1},\Sigma_{1}\right)$
and $\left(\varOmega_{2},\Sigma_{2}\right)$ be measurable spaces
and $f\colon\varOmega_{1}\rightarrow\varOmega_{2}$ be a function.
Function $f$ is called $\Sigma_{1}-\Sigma_{2}$ measurable if for
every measurable set $B\in\Sigma_{2}$, its inverse image is a measurable
set in $\left(\varOmega_{1},\Sigma_{1}\right)$, \LyXThinSpace i.e.,
$f^{-1}\left(B\right)\in\Sigma_{1}$.\end{defn}
\begin{cor}
Let $\left(\varOmega_{1},\Sigma_{1}\right)$, $\left(\varOmega_{2},\Sigma_{2}\right)$
and $\left(\varOmega_{3},\Sigma_{3}\right)$ be measurable spaces
and $f\colon\varOmega_{1}\rightarrow\varOmega_{2}$ be $\Sigma_{1}-\Sigma_{2}$
measurable and $g\colon\varOmega_{2}\rightarrow\varOmega_{3}$ be
$\Sigma_{2}-\Sigma_{3}$ measurable then $g\circ f\colon\varOmega_{1}\rightarrow\varOmega_{3}$
is $\Sigma_{1}-\Sigma_{3}$ measurable.\end{cor}
\begin{proof}
If $A\in\Sigma_{3}$ then $g^{-1}\left(A\right)\in\Sigma_{2}$ ($g$
is $\Sigma_{2}-\Sigma_{3}$ measurable), therefore $f^{-1}\left(g^{-1}\left(A\right)\right)\in\Sigma_{1}$
($f^{-1}$ is $\Sigma_{1}-\Sigma_{2}$ measurable). Clearly, $f^{-1}\left(g^{-1}\left(A\right)\right)=\left(g\circ f\right)^{-1}\left(A\right)\in\Sigma_{1}$
(def inverse). In this way, $g\circ f$ is $\Sigma_{1}-\Sigma_{3}$
measurable.\end{proof}
\begin{defn}
(\textbf{isomorphism of measurable spaces}) Let $\left(\varOmega_{1},\Sigma_{1}\right)$
and $\left(\varOmega_{2},\Sigma_{2}\right)$ be measurable spaces
and $\varphi\colon\varOmega_{1}\rightarrow\varOmega_{2}$ be a bijective
function. $\varphi$ is called $\left(\varOmega_{1},\Sigma_{1}\right)$-$\left(\varOmega_{2},\Sigma_{2}\right)$
isomorphism ($\left(\varOmega_{1},\Sigma_{1}\right)$ and $\left(\varOmega_{2},\Sigma_{2}\right)$
are called isomorphic) if $\varphi$ is $\Sigma_{1}-\Sigma_{2}$ measurable
and $\varphi^{-1}$ is $\Sigma_{2}-\Sigma_{1}$ measurable. 
\end{defn}
After that, we consider measure functions, functions that quantify,
in some way, how much observable events are. This concept of measure
fuction is the starting point on defining probability measure functions.
In the following, we just write $f$ is measurable instead of $f$
is $\Sigma_{1}-\Sigma_{2}$ measurable if the associated $\sigma$-algebras
can be inferred from the context. 
\begin{defn}
(\textbf{measure function}) Let $\left(\varOmega,\Sigma\right)$ be
a measurable space and $\mu:\Sigma\rightarrow\mathbb{\overline{R}}$
be a function from $\Sigma$ to the extended reals ($\overline{\mathbb{R}}=\mathbb{R}\bigcup\left\{ -\infty,\infty\right\} $).
Function $\mu$ is called measure if it satisfies the following three
conditions:
\begin{enumerate}
\item[\noun{$\mu.1$}]  (\textbf{nullity}) $\mu\left(\textrm{Ø}\right)=0$.
\item[\noun{$\mu.2$}]  (\textbf{non-negativity}) $\mu\left(A\right)\geq0$ for all $A\in\Sigma$.
\item[\noun{$\mu.3$}]  (\textbf{$\sigma$-additivity}) $\mu\left(\bigcup_{i\in I}A_{i}\right)=\sum_{i\in I}\mu\left(A_{i}\right)$
for all $\left\{ A_{i}\in\Sigma\right\} _{i\in I}$ \noun{$\mbox{\textsc{cdf}}$}. 
\end{enumerate}
\end{defn}
Then, we consider probability measure functions, functions that quantify,
how probable observable events are. For \noun{SGoal}s, a probability
function will quantify how probable a set of populations can be generated
using stochastic methods.
\begin{defn}
Let $\left(\varOmega,\Sigma\right)$ be a measurable space and $\mu:\Sigma\rightarrow\mathbb{\overline{R}}$
be a measure.
\begin{enumerate}
\item (\textbf{finite measure}) $\mu$ is a finite measure if $\mu\left(A\right)<\infty$
for all $A\in\Sigma$.
\item (\textbf{probability measure}) $\mu$ is a probability measure if
$\mu\left(\varOmega\right)=1$. 
\end{enumerate}
\end{defn}
Now, we are ready to define the mathematical structure used by probability
theory.
\begin{defn}
If $\left(\varOmega,\Sigma\right)$ is a measurable space and $\mu:\Sigma\rightarrow\mathbb{\overline{R}}$
is a measure function
\begin{enumerate}
\item (\textbf{measure space}) $\left(\varOmega,\Sigma,\mu\right)$ is called
measure space.
\item (\textbf{probability space}) If $\mu$ is a probability measure, $\left(\varOmega,\Sigma,\mu\right)$
is called probability space.
\end{enumerate}
\end{defn}
Finally, we can define the concept of random variable, a function
that preserves observable events and quantifies how probable an observable
event is.
\begin{defn}
(\textbf{random variable}) Let $\left(\varOmega_{1},\Sigma_{1},Pr\right)$
be a probability space and $\left(\varOmega_{2},\Sigma_{2}\right)$
be a measurable space. If $X\colon\varOmega_{1}\rightarrow\varOmega_{2}$
is a measurable function then $X$ is called a random variable with
values in $\left(\varOmega_{2},\Sigma_{2}\right)$. 
\end{defn}
For $A\in\Sigma_{2}$ , we denote $\left\{ X\in A\right\} \equiv X^{-1}(A)$
and $Pr\left[X\in A\right]\equiv Pr\left[X^{-1}\left(A\right)\right]$.
In particular, if $\left(\varOmega_{2},\Sigma_{2}\right)=\left(\mathbb{R},\mathcal{B}\left(\mathbb{R}\right)\right)$,
then $X$ is called a real random variable and we let $Pr\left[X\geq0\right]\equiv X^{-1}\left(\left[0,\infty\right)\right)$.

\subsection{\label{sub:Kernel}Kernel}

As pointed by Breiman in Section 4.3 of \cite{Breiman68}, the kernel
is a regular conditional probability $K\left(x,A\right)=P(x,A)=Pr\left[X_{t}\in A\mid X_{t-1}=x\right]$.
For \noun{SGoal}s, a kernel will be used for characterizing the stochastic
process carried on iteration by iteration (generating a population
from a population).
\begin{defn}
(\textbf{Markov kernel}) \label{def:Kernel}Let $\left(\varOmega_{1},\Sigma_{1}\right)$
and $\left(\varOmega_{2},\Sigma_{2}\right)$ be measurable spaces.
A function $K:\varOmega_{1}\times\Sigma_{2}\rightarrow\left[0,1\right]$
is called a (Markov) kernel if the following two conditions hold:
\begin{enumerate}
\item[\noun{$K.1$}]  Function $K_{x,\bullet}\vcentcolon A\mapsto K(x,A)$ is a probability
measure for each fixed $x\in\varOmega_{1}$ and 
\item[\noun{$K.2$}]  Function $K_{\bullet,A}\vcentcolon x\mapsto K(x,A)$ is a measurable
function for each fixed $A\in\Sigma_{2}$. 
\end{enumerate}
\end{defn}
\begin{rem}
\label{rem:MeasurableKernelPi}As noticed by Kenkle in Remark 8.26,
page 181 of \cite{Kenkle14}, it is sufficient to check $K.2$ in
definition \ref{def:Kernel} for sets $A$ from a $\pi$-system $\mathcal{E}$
that generates $\Sigma_{2}$ and that either contains $A$ or a sequence
$A_{n}\uparrow A$ . Indeed, in this case, $\mathcal{D}=\left\{ A\in\Sigma_{2}\mid K_{\bullet,A}^{-1}\in\Sigma_{1}\right\} $
is a $\lambda$-system. Since $\mathcal{E}\subset\mathcal{D}$, by
the Dynkin $\pi-\lambda$ theorem (\ref{thm:Dynkin---theorem.}),
$\mathcal{D}=\sigma\left(\mathcal{E}\right)=\Sigma_{2}$.
\end{rem}
If the transition density $K\colon\varOmega_{1}\times\varOmega_{2}\rightarrow\left[0,1\right]$
exits, then the transition kernel can be defined using equation \ref{eq:kernel-density}.
In the rest of this paper, we will consider kernels having transition
densities. 

\begin{equation}
K\left(x,A\right)=\int_{A}K\left(x,y\right)dy\label{eq:kernel-density}
\end{equation}

Kernels that will play a main role in a systematic development of
a formal theory for \noun{SGoal}s are those associated to deterministic
methods that are used by a particular \noun{Sgoal}, like selecting
any/the best individual in a population, or sorting a population.
Theorem \ref{thm:Kernel_det_f} provides a sufficient condition for
characterizing deterministic methods as kernels.
\begin{thm}
\label{thm:Kernel_det_f}(\textbf{deterministic kernel}) Let $\left(\varOmega_{1},\Sigma_{1}\right)$
and $\left(\varOmega_{2},\Sigma_{2}\right)$ be measurable spaces,
and $f\vcentcolon\varOmega_{1}\rightarrow\varOmega_{2}$ be $\Sigma_{1}-\Sigma_{2}$
measurable. The function $1_{f}\vcentcolon\varOmega_{1}\times\Sigma_{2}\rightarrow\left[0,1\right]$
defined as follow is a kernel.

\[
1_{f}\left(x,A\right)=\left\{ \begin{array}{ll}
1 & \mbox{if }f\left(x\right)\in A\\
0 & \mbox{otherwise}
\end{array}\right.
\]

\end{thm}
\begin{proof}
\textbf{{[}well-defined{]}} Obvious, it is defined using the membership
predicate and takes only two values $0$ and $1$. \textbf{{[}$K.1${]}}
Let $x\in\varOmega_{1}$, clearly $1_{f\,x,\bullet}\left(A\right)\geq0$
for all $A\in\Sigma_{2}$ so $1_{f\,x,\bullet}$ is non-negative.
Now, $1_{f}\left(x,\textrm{Ø}\right)=0$ so it satisfies $\mu.1$.
Let $\left\{ A_{i}\in\Sigma_{2}\right\} _{i\in I}$ be a \noun{$\mbox{\textsc{cdf}}$}.
If $x\in\biguplus_{i\in I}A_{i}$ then $1_{f}\left(x,\biguplus_{i\in I}A_{i}\right)=1$
(def $1_{f}\left(x,A\right)$), and $\exists k\in I$ such that $x\in A_{k}$
(def $\biguplus$). Therefore, $x\notin A_{i}$ for all $i\neq k\in I$
($\left\{ A_{i}\in\Sigma\right\} _{i\in I}$ is \noun{$\mbox{\textsc{df}}$})
so $1_{f}\left(x,A_{k}\right)=1$ and $1_{f}\left(x,A_{i}\right)=0$
for all $i\neq k\in I$ (def $1_{f}\left(x,A\right)$). Clearly, $1_{f\,x,\bullet}\left(\biguplus_{i\in I}A_{i}\right)=1=\sum_{i\in I}1_{f\,x,\bullet}\left(A_{i}\right)$.
A similar proof is carried on when $x\notin\biguplus_{i\in I}A_{i}$,
in this case $1_{x,\bullet}\left(\biguplus_{i\in I}A_{i}\right)=0=\sum_{i\in I}1_{f\,x,\bullet}\left(A_{i}\right)$.
Then, $1_{f\,x,\bullet}$ is \textbf{$\sigma$}-additive, so it is
a measure. $1_{f}\left(x,\varOmega_{2}\right)=1$ (obvious) then $1_{f\,x,\bullet}$
is a probability measure for a fixed $x\in\varOmega_{1}$. \textbf{{[}$K.2${]}}
Let $A\in\Sigma_{2}$ and $\alpha\in\mathbb{Q}^{+}$. If $\alpha<1$
then $1_{f\,\bullet,A}^{-1}\left(\left(0,\alpha\right]\right)=\textrm{Ø}$
($1_{f}\left(x,A\right)=\left\{ 0,1\right\} \notin\left(0,\alpha\right]$),
so $1_{f\,\bullet,A}^{-1}\left(\left(0,\alpha\right]\right)\in\Sigma_{1}$
(lemma \ref{lem:sigma-intersections}.1). Now, if$\alpha\geq1$ then
$1_{f\,\bullet,A}^{-1}\left(\left(0,\alpha\right]\right)=\left\{ x\in\varOmega_{1}\mid1_{f}\left(x,A\right)=1\in\left(0,\alpha\right]\right\} $
($1_{f}\left(x,A\right)=1$ is the only value in $\left(0,\alpha\right]$),
i.e., $1_{f\,\bullet,A}^{-1}\left(\left(0,\alpha\right]\right)=\left\{ x\in\varOmega_{1}\mid f\left(x\right)\in A\right\} =f^{-1}\left(A\right)$
(def $f^{-1}$), so $1_{f\,\bullet,A}^{-1}\left(\left(0,\alpha\right]\right)\in\Sigma_{1}$
($f$ is measurable). Therefore, $1_{f\,\bullet,A}$ is measurable. \end{proof}
\begin{cor}
\label{lem:IdKernel}(\textbf{Indicator kernel}) Let $\left(\varOmega,\Sigma\right)$
be a measurable space. The indicator function $1\vcentcolon\varOmega\times\Sigma\rightarrow\left[0,1\right]$
defined as $1\left(x,A\right)=1_{id\left(x\right)}\left(A\right)$,
with $id\left(x\right)=x$ is a kernel.\end{cor}
\begin{proof}
According to theorem \ref{thm:Kernel_det_f}, it is sufficient to
prove that $id$ is a measurable function. It is obvious, we have
that $A=id\left(A\right)=id^{-1}\left(A\right)$ (def $id$) then
$id^{-1}\left(A\right)\in\Sigma$ if $A\in\Sigma$.
\end{proof}
Transition probabilities (kernels) also represent linear operators
over infinite-dimensional vector spaces \cite{Geyer05,Geyer11}. Therefore,
operations like kernels multiplication, and kernels convex combinations
can be used in order to preserve the Markovness property of the resulting
transition kernel (sometimes called update mechanism).

\subsubsection{\label{sub:Random-Scan-(Mixing)}Random Scan (Mixing)}

The random scan (mixing) update mechanism follows the idea of picking
one update mechanism (among a collection of predefined update mechanisms)
and then applying it. Such update mechanism is picked according to
some weigth associated to each one of the update mechanism. Following
this idea, the mixing update mechanism is built using kernels addition
and kernel multiplication by a scalar.

In order to maintain the Markovness property (both operations, kernels
addition and kernel multiplication by a scalar, in general, do not
preserve such property), a convex combination of them is considered. 
\begin{defn}
\label{def:(mixing)}(\textbf{mixing}) The mixing update mechanism
of a set of $n$ Markov transition kernels $K_{1}$, ..., $K_{n}$,
each of them with a probability of being picked $p_{1},p_{2},\ldots,p_{n}$
($\sum p_{i}=1$), is defined by equation \ref{eq:mixingKernels}.
\end{defn}
\begin{equation}
\left(\sum_{i=1}^{n}p_{i}K_{1}\right)\left(x,A\right)=\int_{A}\sum_{i=1}^{n}p_{i}K_{i}\left(x,y\right)dy\label{eq:mixingKernels}
\end{equation}

Since the integral in equation \ref{eq:mixingKernels} is a linear
operator, the mixing operation can be defined by equation \ref{eq:mixingKernels-1}.

\begin{equation}
\left(\sum_{i=1}^{n}p_{i}K_{1}\right)\left(x,A\right)=\sum_{i=1}^{n}p_{i}\int_{A}K_{i}\left(x,y\right)dy\label{eq:mixingKernels-1}
\end{equation}

\subsubsection{\label{sub:Composition}Composition}

The composition update mechanism follows the idea of applying an update
mechanism (kernel) followed by other update mechanism and so on. Following
this idea, the composition update mechanism is built using the kernel
multiplication operator.
\begin{defn}
(\textbf{composition}) The composition of two kernels $K_{1}$, $K_{2}$
is defined by equation \ref{eq:multiplyKernels}.
\end{defn}
\begin{equation}
\left(K_{2}\circ K_{1}\right)\left(x,A\right)=\int K_{2}\left(y,A\right)K_{1}\left(x,dy\right)\label{eq:multiplyKernels}
\end{equation}

Since the kernel multiplication is an associative operation (using
the conditional Fubini theorem, see Theorem 2 of Chapter 22, page
431 of the book of Fristedt and Gray \cite{FristedtGray97}), the
composition of update mechanisms that corresponds to a set of $n$
transition kernels $K_{1}$, ..., $K_{n}$ is defined as the product
kernel $K_{n}\circ K_{n-1}\circ\ldots\circ K_{1}$.

\subsubsection{Transition's Kernel Iteration}

The transition probability $t$-th iteration (application) of a Markovian
kernel $K$, given by equation \ref{eq:Kernel-Iteration}, describes
the probability to transit to some set $A\in\Sigma$ within $t$ steps
when starting at state $x\in\varOmega$.

\begin{equation}
K^{\left(t\right)}\left(x,A\right)=\begin{cases}
K\left(x,A\right) & ,\,t=1\\
{\displaystyle \intop_{\varOmega}}K^{\left(t-1\right)}\left(y,A\right)K\left(x,dy\right) & ,\,t>1
\end{cases}\label{eq:Kernel-Iteration}
\end{equation}

If $p:\Sigma\rightarrow\left[0,1\right]$ is the initial distribution
of subsets, then the probability that the Markov process is in set
$A\in\Sigma$ at step $t\geq0$ is given by equation \ref{eq:kernel-transition2}.

\begin{equation}
Pr\left\{ X_{t}\in A\right\} =\begin{cases}
p\left(A\right) & ,\,t=0\\
{\displaystyle \intop_{\varOmega}}K^{\left(t\right)}\left(x,A\right)p\left(dx\right) & ,\,t>0
\end{cases}\label{eq:kernel-transition2}
\end{equation}

\subsection{\label{sub:Markov-Chains}Markov Chains}
\begin{defn}
(\textbf{Markov chain}) A discrete-time stochastic process $X_{0}$,
$X_{1}$, $X_{2}$, $\ldots$, taking values in an arbitrary state
space $\varOmega$ is a Markov chain if it satisfies: 
\begin{enumerate}
\item (\textbf{Markov property}) The conditional distribution of $X_{t}$
given $X_{0}$, $X_{1}$,$\ldots$, $X_{t-1}$ is the same as the
conditional distribution of $X_{t}$ given only $X_{t-1}$,
\item (\textbf{stationarity property}) The conditional distribution of $X_{t}$
given $X_{t-1}$ does not depend on $t$. 
\end{enumerate}
\end{defn}
Clearly, transition probabilities of the chain are specified by the
conditional distribution of $X_{t}$ given $X_{t-1}$ (kernel), while
the probability law of the chain is completely specified by the initial
distribution $X_{0}$. Moreover, many \noun{SGoal}s may be characterized
by Markov chains.

\subsection{\label{sub:Convergence}Convergence}
\begin{defn}
Let $\left(D_{t}\right)$ be a random sequence, i.e., a sequence of
random variables defined on a probability space $\left(\varOmega,\Sigma,P\right)$.
Then $\left(D_{t}\right)$ is said to 
\begin{enumerate}
\item \textbf{Converge completely to zero}, denoted as $D_{t}\overset{c}{\rightarrow}0$,
if equation \ref{eq:completely-0} holds for every $\epsilon>0$

\begin{equation}
\underset{t\rightarrow\infty}{\lim}{\displaystyle \sum_{i=1}^{t}Pr\left\{ \left|D_{t}\right|>\epsilon\right\} }<\infty\label{eq:completely-0}
\end{equation}

\item \textbf{Converge in probability to zero}, denoted as $D_{t}\overset{p}{\rightarrow}0$,
if equation \ref{eq:probability-0} holds for every $\epsilon>0$.

\begin{equation}
\underset{t\rightarrow\infty}{\lim}{\displaystyle Pr\left\{ \left|D_{t}\right|>\epsilon\right\} }=0\label{eq:probability-0}
\end{equation}

\end{enumerate}
\end{defn}
Notice that convergence in probability to zero (equation \ref{eq:probability-0})
is a necessary condition for convergence completely to zero (equation
\ref{eq:completely-0}).

\section{Probability Theory on Cartesian Products}

Since we are working on populations (finite tuples of individuals
in the search space), we need to consider probability theory on generalized
cartesian products of the search space (subsection \ref{sub:Generalized-Cartesian-Product}).
By considering some mathematical properties of the generalized cartesian
product (we will move from tuples of tuples to just a single tuple),
some mathematical proofs can be simplified and an appropiated $\sigma$-algebra
for populations can be defined (subsection \ref{sub:Product-sigma-algebra}).
These accesory definitions, propositions and theorems will allow us
to define a kernel by joining some simple kernels (section \ref{sub:Kernel-join}).
Therefore, we will be able to work with stochastical methods in \noun{SGoal}s
that are defined as joins of stochastic methods that produce subpopulations
(sub-tuples) of the newly generated complete population (single tuple).

\subsection{\label{sub:Generalized-Cartesian-Product}Generalized Cartesian Product}
\begin{defn}
\label{def:Cartesian Product}(\textbf{cartesian product}) Let $\mathcal{L}=\left\{ \varOmega_{1},\varOmega_{2},...,\varOmega_{n}\right\} $
be an ordered list of $n\in\mathbb{N}$ sets. The Cartesian product
of $\mathcal{L}$ is the set of ordered $n$-tuples: $\prod_{i=1}^{n}\varOmega_{i}=\left\{ \left(a_{1},a_{2},\ldots,a_{n}\right)\mid a_{i}\in\varOmega_{i}\mbox{ for all }i=1,2,\ldots,n\right\} $.
If $\varOmega=\varOmega_{i}$ for all $i=1,2,\ldots,n$ then $\prod_{i=1}^{n}\varOmega_{i}$
is noted $\varOmega^{n}$ and it is called the \textbf{$n$}-fold
cartesian product of set $\varOmega$.\end{defn}
\begin{lem}
\label{lem:asso-comm-product}Let $\varOmega_{1}$, $\varOmega_{2}$
and $\varOmega_{3}$ be sets, then
\begin{enumerate}
\item (\textbf{associativity}) $\left(\varOmega_{1}\times\varOmega_{2}\right)\times\varOmega_{3}\equiv\varOmega_{1}\times\left(\varOmega_{2}\times\varOmega_{3}\right)\equiv\varOmega_{1}\times\varOmega_{2}\times\varOmega_{3}$.
\item (\textbf{commutativity}) $\varOmega_{1}\times\varOmega_{2}\equiv\varOmega_{2}\times\varOmega_{1}$
\end{enumerate}
\end{lem}
\begin{proof}
\textbf{{[}1{]}} Functions $h_{L}\vcentcolon\left(\varOmega_{1}\times\varOmega_{2}\right)\times\varOmega_{3}\rightarrow\varOmega_{1}\times\varOmega_{2}\times\varOmega_{3}$
and $h_{R}\vcentcolon\varOmega_{1}\times\left(\varOmega_{2}\times\varOmega_{3}\right)\rightarrow\varOmega_{1}\times\varOmega_{2}\times\varOmega_{3}$
such that $h_{L}\left(\left(a,b\right),c\right)=\left(a,b,c\right)$
and $h_{R}\left(a,\left(b,c\right)\right)=\left(a,b,c\right)$ are
equivalence functions. \textbf{{[}2{]} }Function $r\vcentcolon A\times B\rightarrow B\times A$
for any $A,B$ such that $r\left(a,b\right)=\left(b,a\right)$ is
a bijective function.\end{proof}
\begin{cor}
Let $\left\{ n_{i}\in\mathbb{N}^{+}\right\} _{i=1,2,\ldots,m}$ be
an ordered list of $m\in\mathbb{N}$ positive natural numbers.
\begin{enumerate}
\item $\prod_{i=1}^{m}\prod_{j=1}^{n_{i}}\varOmega_{i,j}\equiv\varOmega_{1,1}\times\ldots\times\varOmega_{1,n_{1}}\times\ldots\times\varOmega_{m,1}\times\ldots\times\varOmega_{m,n_{m}}$
with $\varOmega_{i,j}$ a set for all $i=1,2,\ldots m$ and $j=1,2,\ldots,n_{i}$.
\item $\prod_{i=1}^{m}\varOmega^{n_{i}}\equiv\varOmega^{n}$ with $n=\sum_{i=1}^{m}n_{i}$.
\end{enumerate}
\end{cor}

\subsection{\label{sub:Product-sigma-algebra}Product $\sigma$-algebra}

Products $\sigma$-algebra allow us to define appropiated $\sigma$-algebra
for generalized cartesian products. If we are provided with a $\sigma$-algebra
associated to the feasible region of a \noun{SGoal}, then we can define
a $\sigma$-algebra for populations of it.
\begin{defn}
\label{def:SigmaProduct}Let $\mathcal{L}=\left\{ \Sigma_{1},\Sigma_{2},\ldots,\Sigma_{n}\right\} $
be a $n\in\mathbb{N}$ ordered list of $\Sigma_{i}\subseteq2^{\varOmega_{i}}$
family of sets. 
\begin{enumerate}
\item \textbf{(generalized family product) }The generalized product of $\mathcal{L}$
is $\prod_{i=1}^{n}\Sigma_{i}=\left\{ \prod_{i=1}^{n}A_{i}\mid\forall_{i=1}^{n}A_{i}\in\Sigma_{i}\right\} $
\item (\textbf{product $\sigma$-algebra}) If $\Sigma_{i}$ is a $\sigma$-algebra
for all $i=1,2,\ldots,n$, then the product $\sigma$-algebra of $\mathcal{L}$
is the $\sigma$-algebra $\bigotimes{}_{i=1}^{n}\Sigma_{i}=\sigma\left(\prod_{i=1}^{n}\Sigma_{i}\right)$
defined over the set $\prod_{i=1}^{n}\varOmega_{i}$.
\end{enumerate}
\end{defn}
\begin{lem}
\label{lem:-sigma-product}If $\mathcal{L}=\left\{ \left(\Sigma_{1},\varOmega_{1}\right),\left(\Sigma_{2},\varOmega_{2}\right),\ldots,\left(\Sigma_{n},\varOmega_{n}\right)\right\} $
is a finite ($n\in\mathbb{N}$) ordered list of measurable spaces
then $\prod_{i=1}^{n}\Sigma_{i}$ is a $\pi$-system.\end{lem}
\begin{proof}
If $U,V\in\prod_{i=1}^{n}\Sigma_{i}$ then for all $i=1,2,\ldots,n$
exist $U_{i},V_{i}\in\Sigma_{i}$ such that $U=\prod_{i=1}^{n}U_{i}$
and $V=\prod_{i=1}^{n}V_{i}$ (def. $\prod_{i=1}^{n}\Sigma_{i}$).
Clearly, $U_{i}\bigcap V_{i}\in\Sigma_{i}$ (lemma \ref{lem:sigma-intersections}).
Therefore, $\prod_{i=1}^{n}\left(U_{i}\bigcap V_{i}\right)\in\prod_{i=1}^{n}\Sigma_{i}$
(def. $\prod_{i=1}^{n}\Sigma_{i}$). Let $z=\left(z_{1},z_{2},\ldots,z_{n}\right)\in\prod_{i=1}^{n}\varOmega_{i}$.
Clearly, $z\in U\bigcap V$ iff $z\in U$ and $z\in V$ (def $\bigcap$)
iff $z_{i}\in U_{i}$ and $z_{i}\in V_{i}$ for all $i=1,2,\ldots,n$
(def $U$ and $V$) iff $z_{i}\in U_{i}\bigcap V_{i}$ for all $i=1,2,\ldots,n$
(def $\bigcap$) iff $z\in\prod_{i=1}^{n}\left(U_{i}\bigcap V_{i}\right)$
(def $\prod$). Therefore, $U\bigcap V=\prod_{i=1}^{n}\left(U_{i}\bigcap V_{i}\right)\in\prod_{i=1}^{n}\Sigma_{i}$.
\end{proof}
Proposition \emph{\ref{prop:(associativity-of-sigma-algebra} }will
allow us to move from the product $\sigma$-algebra of products $\sigma$-algebras
to a single product $\sigma$-algebra (as we move from tuples of tuples
to just a single tuple).
\begin{prop}
\label{prop:(associativity-of-sigma-algebra}(\textbf{associativity
of $\sigma$-algebra} \textbf{product}) Let $\Sigma_{i}$ be a $\sigma$-algebra
defined over a set $\varOmega_{i}$ for all $i=1,2,3$, then $\left(\Sigma_{1}\otimes\Sigma_{2}\right)\otimes\Sigma_{3}\equiv\Sigma_{1}\otimes\left(\Sigma_{2}\otimes\Sigma_{3}\right)\equiv\Sigma_{1}\otimes\Sigma_{2}\otimes\Sigma_{3}$.\end{prop}
\begin{proof}
\textbf{{[}$\left(\Sigma_{1}\otimes\Sigma_{2}\right)\otimes\Sigma_{3}\equiv\Sigma_{1}\otimes\Sigma_{2}\otimes\Sigma_{3}${]}
{[}$\subseteq${]}} We will use the Dynkin $\pi-\lambda$ theorem
\ref{thm:Dynkin---theorem.} here. So, we need to find a $\pi$-system
that is contained by a $\lambda$ system. Consider $A_{3}\in\Sigma_{3}$
and define $\mathcal{A}_{A_{3}}=\left\{ X\in\Sigma_{1}\otimes\Sigma_{2}\mid X\times A_{3}\in\sigma\left(\Sigma_{1}\times\Sigma_{2}\times\Sigma_{3}\right)\right\} $.
{[}$\pi$\textbf{-system}{]} $A_{1}\times A_{2}\times A_{3}\in\Sigma_{1}\times\Sigma_{2}\times\Sigma_{3}$
for all $A_{1}\in\Sigma_{1}$ and $A_{2}\in\Sigma_{2}$ ($A_{3}\in\Sigma_{3}$)
then $\Sigma_{1}\times\Sigma_{2}\subset\mathcal{A}_{A_{3}}$.$\left(\Sigma_{1}\otimes\Sigma_{2}\right)\times\Sigma_{3}$.
{[}$\lambda$\textbf{-system}{]} \textbf{{[}$\lambda.1${]}} $\varOmega_{i}\in\Sigma_{i}$
for $i=1,2,3$ ($\Sigma_{i}$ $\sigma$-algebra) then $\varOmega_{1}\times\varOmega_{2}\times A_{3}\in\Sigma_{1}\times\Sigma_{2}\times\Sigma_{3}$.
Therefore, $\varOmega_{1}\times\varOmega_{2}\in\mathcal{A}_{A_{3}}$
(def $\mathcal{A}_{A_{3}}$). \textbf{{[}$\lambda.2${]}} Let $X\in\mathcal{A}_{A_{3}}$
then $X^{c}\times A_{3}=\left(\left(\varOmega_{1}\times\varOmega_{2}\right)\setminus X\right)\times A_{3}$
(def. complement). Clearly, $X^{c}\times A_{3}=\left(\left(\varOmega_{1}\times\varOmega_{2}\right)\times A_{3}\right)\bigcap\left(X\times A_{3}\right)^{c}$
(Distribution and Morgan's law). Now, $\left(X\times A_{3}\right)\in\sigma\left(\Sigma_{1}\times\Sigma_{2}\times\Sigma_{3}\right)$
(def $\mathcal{A}_{A_{3}}$), then $\left(X\times A_{3}\right)^{c}\in\sigma\left(\Sigma_{1}\times\Sigma_{2}\times\Sigma_{3}\right)$
($\sigma$-algebra). Moreover, $\left(\varOmega_{1}\times\varOmega_{2}\right)\times A_{3}\in\sigma\left(\Sigma_{1}\times\Sigma_{2}\times\Sigma_{3}\right)$
(part 1 and def $\mathcal{A}_{A_{3}}$) therefore, $X^{c}\times A_{3}\in\sigma\left(\Sigma_{1}\times\Sigma_{2}\times\Sigma_{3}\right)$
(lemma \ref{lem:sigma-intersections}.2). So, $X^{c}\in\mathcal{A}_{A_{3}}$.
\textbf{{[}$\lambda.3${]}} Let $\left\{ X_{i}\right\} _{i\in I}\subseteq\mathcal{A}_{A_{3}}$
be \noun{$\mbox{\textsc{cdf}}$} of $\mathcal{A}_{A_{3}}$. Clearly,
$\bigcup_{i\in I}\left(X_{i}\times A_{3}\right)=\left(\bigcup_{i\in I}X_{i}\right)\times A_{3}$
(sets algebra), and $\left(\bigcup_{i\in I}X_{i}\right)\in\Sigma_{1}\otimes\Sigma_{2}$
($X_{i}\in\Sigma_{1}\otimes\Sigma_{2}$ for all $i\in I$). then,
$\bigcup_{i\in I}\left(X_{i}\times A_{3}\right)\in\mathcal{A}_{A_{3}}$.
Therefore, $\mathcal{A}_{A_{3}}$ is a $\lambda$-system. In this
way, $\Sigma_{1}\otimes\Sigma_{2}=\sigma\left(\Sigma_{1}\times\Sigma_{2}\right)\subseteq\mathcal{A}_{A_{3}}$
(Dynkin $\pi-\lambda$ theorem \ref{thm:Dynkin---theorem.}), and
$\sigma\left(\Sigma_{1}\times\Sigma_{2}\right)\times A_{3}\subseteq\sigma\left(\Sigma_{1}\times\Sigma_{2}\times\Sigma_{3}\right)$
(def $\mathcal{A}_{A_{3}}$), i.e., $\left(\Sigma_{1}\otimes\Sigma_{2}\right)\times A_{3}\subseteq\Sigma_{1}\otimes\Sigma_{2}\otimes\Sigma_{3}$.
Because, $\left(\Sigma_{1}\otimes\Sigma_{2}\right)\times A_{3}\subseteq\Sigma_{1}\otimes\Sigma_{2}\otimes\Sigma_{3}$
for all $A_{3}\in\Sigma_{3}$ then $\left(\Sigma_{1}\otimes\Sigma_{2}\right)\times\Sigma_{3}\subseteq\Sigma_{1}\otimes\Sigma_{2}\otimes\Sigma_{3}$
and $\sigma\left(\left(\Sigma_{1}\otimes\Sigma_{2}\right)\times\Sigma_{3}\right)=\left(\Sigma_{1}\otimes\Sigma_{2}\right)\otimes\Sigma_{3}\subseteq\Sigma_{1}\otimes\Sigma_{2}\otimes\Sigma_{3}$
(def $\sigma\left(\cdot\right)$). \textbf{{[}$\supseteq${]}} It
is clear that $\Sigma_{1}\times\Sigma_{2}\subseteq\Sigma_{1}\otimes\Sigma_{2}$
(def $\otimes$), so $\Sigma_{1}\times\Sigma_{2}\times\Sigma_{3}\subseteq\left(\Sigma_{1}\otimes\Sigma_{2}\right)\times\Sigma_{3}$,
therefore, $\sigma\left(\Sigma_{1}\times\Sigma_{2}\times\Sigma_{3}\right)\subseteq\sigma\left(\left(\Sigma_{1}\otimes\Sigma_{2}\right)\times\Sigma_{3}\right)$
(def $\sigma\left(\cdot\right))$, i.e., $\left(\Sigma_{1}\otimes\Sigma_{2}\right)\otimes\Sigma_{3}\supseteq\Sigma_{1}\otimes\Sigma_{2}\otimes\Sigma_{3}$.
{[}\textbf{$\Sigma_{1}\otimes\left(\Sigma_{2}\otimes\Sigma_{3}\right)\equiv\Sigma_{1}\otimes\Sigma_{2}\otimes\Sigma_{3}$}{]}
A similar proof to the $\left(\Sigma_{1}\otimes\Sigma_{2}\right)\otimes\Sigma_{3}\equiv\Sigma_{1}\otimes\Sigma_{2}\otimes\Sigma_{3}$
is carried on. \end{proof}
\begin{cor}
If $\Sigma$ is a $\sigma$-algebra defined over set $\varOmega$
and $\left\{ \bigotimes_{k=1}^{n_{i}}\Sigma\right\} _{i=1,2,\ldots,m}$
is an ordered lists of the $m\in\mathbb{N}$ given product $\sigma$-algebras
($n_{i}\in\mathbb{N}^{+}$ for all $i=1,2,\ldots,m$) then $\bigotimes_{i=1}^{n}\Sigma\equiv\bigotimes_{i=1}^{m}\left(\bigotimes_{k=1}^{n_{i}}\Sigma\right)$
with $n=\sum_{i=1}^{m}n_{i}$.
\end{cor}
In the rest of this paper, we will denote $\prod_{i=1}^{n}\mathcal{A}\equiv\mathcal{A}^{n}$
for any $\mathcal{A}\subseteq2^{\varOmega}$ , and $\Sigma^{\otimes n}\equiv\bigotimes_{i=1}^{n}\Sigma$
for any $\sigma$-algebra $\Sigma$ on $\Omega$.

\subsection{\label{sub:Kernel-join}Kernels on product $\sigma$-algebras}

Now, we are in the position of defining a kernel that characterizes
a deterministic method that is commonly used by \noun{SGoal}s (as
part of individual's selection methods based on fitness): the \noun{Swap}
method. 
\begin{defn}
\label{def:swap-two}(\textbf{\noun{Swap}}) Let $\varOmega_{1}$ and
$\varOmega_{2}$ be two sets. The swap function ($\underleftarrow{}$)
is defined as follows\footnote{We will use the notation $\underleftarrow{z}\equiv\underleftarrow{}\left(z\right)$
for any $z\in\varOmega_{1}\times\varOmega_{2}$ and $\underleftarrow{A}\equiv\underleftarrow{}\left(A\right)$
for any $A\subseteq\varOmega_{1}\times\varOmega_{2}$.}:

\[
\begin{array}{cccc}
\underleftarrow{}\vcentcolon & \varOmega_{1}\times\varOmega_{2} & \rightarrow & \varOmega_{2}\times\varOmega_{1}\\
 & \left(x,y\right) & \mapsto & \left(y,x\right)
\end{array}
\]

\end{defn}
\begin{lem}
\label{lem:Swap}Let $\varOmega_{1}$ and $\varOmega_{2}$ be two
sets.
\begin{enumerate}
\item $\textrm{Ø}=\underleftarrow{\textrm{Ø}}$ and $\underleftarrow{\varOmega_{1}\times\varOmega_{2}}=\varOmega_{2}\times\varOmega_{1}$
\item $z\in\underleftarrow{A}$ iff $\underleftarrow{z}\in A$
\item $A=\underleftarrow{\underleftarrow{A}}$ for all $A\subseteq\varOmega_{1}\times\varOmega_{2}$
\item $\underleftarrow{B\setminus A}=\underleftarrow{B}\setminus\underleftarrow{A}$
for all $A,B\subseteq\varOmega_{1}\times\varOmega_{2}$
\item $\underleftarrow{\bigcup_{i\in I}A_{i}}=\bigcup_{i\in I}\underleftarrow{A_{i}}$
for any family$\left\{ A_{i}\subseteq\varOmega_{1}\times\varOmega_{2}\right\} _{i\in I}$.
\end{enumerate}
\end{lem}
\begin{proof}
\textbf{{[}1, 2, 3{]}} Are obvious (just applying def swap). \textbf{{[}4{]}}
Let $A,B\subseteq\varOmega_{1}\times\varOmega_{2}$. Now, $\underleftarrow{B\setminus A}=\left\{ z\mid\underleftarrow{z}\in B\setminus A\right\} $
(def swap), so $\underleftarrow{B\setminus A}=\left\{ z\mid\underleftarrow{z}\in B\wedge\underleftarrow{z}\notin A\right\} $
(def proper diff). Clearly, $\underleftarrow{B\setminus A}=\left\{ z\mid z\in\underleftarrow{B}\wedge z\notin\underleftarrow{A}\right\} $
(2), i.e. $\underleftarrow{B\setminus A}=\underleftarrow{B}\setminus\underleftarrow{A}$
(def proper diff). \textbf{{[}5{]}} Let $\left\{ A_{i}\subseteq\varOmega_{1}\times\varOmega_{2}\right\} _{i\in I}$
a family of sets, $z\in\underleftarrow{\bigcup_{i\in I}A_{i}}$ iff
$\exists i\in I$ such that $z\in\underleftarrow{A_{i}}$ iff $z\in\bigcup_{i\in I}\underleftarrow{A_{i}}$.\end{proof}
\begin{prop}
\label{prop:swap}Let $\left(\varOmega_{1},\Sigma_{1}\right)$ and
$\left(\varOmega_{2},\Sigma_{2}\right)$ be measurable spaces then
$A\in\Sigma_{1}\otimes\Sigma_{2}$ iff $\underleftarrow{A}\in\Sigma_{2}\otimes\Sigma_{1}$.\end{prop}
\begin{proof}
\textbf{{[}$\rightarrow${]}} We will apply the Dynkin \textbf{$\pi$-$\lambda$
}theorem (theorem \ref{thm:Dynkin---theorem.}). Let $\mathcal{A}=\left\{ A\in\Sigma_{1}\otimes\Sigma_{2}\mid\underleftarrow{A}\in\Sigma_{2}\otimes\Sigma_{1}\right\} $.
\textbf{{[}$\lambda.1${]}} Obvious, $\textrm{Ø}=\underleftarrow{\textrm{Ø}}\in\Sigma_{2}\otimes\Sigma_{1}$
(lemma \ref{lem:Swap}.1 and lemma \ref{lem:sigma-intersections}.1).
\textbf{{[}$\lambda.2${]}} Let $A,B\in\mathcal{A}$ such that $A\subset B$.
Since $\underleftarrow{B\setminus A}=\underleftarrow{B}\setminus\underleftarrow{A}$
(lemma \ref{lem:Swap}.4), and $\underleftarrow{A},\underleftarrow{B}\in\Sigma_{2}\otimes\Sigma_{1}$
(def $\mathcal{A}$), then $\underleftarrow{B\setminus A}=\underleftarrow{B}\setminus\underleftarrow{A}\in\Sigma_{2}\otimes\Sigma_{1}$
(lemma \ref{lem:sigma-intersections}.2), i.e., $\mathcal{A}$ is
\noun{$\overline{\mbox{\textsc{pd}}}$}. \textbf{{[}$\lambda.3${]}}
Let $\left\{ A_{i}\subseteq\varOmega\right\} _{i\in I}$ a \noun{$\mbox{\textsc{cdf}}$},
$\underleftarrow{\bigcup_{i\in I}A_{i}}=\bigcup_{i\in I}\underleftarrow{A_{i}}$
(lemma \ref{lem:Swap}.5) and $\underleftarrow{A_{i}}\in\Sigma_{2}\otimes\Sigma_{1}$
(def $\mathcal{A}$) then $\underleftarrow{\bigcup_{i\in I}A_{i}}\in\Sigma_{2}\otimes\Sigma_{1}$
($\sigma.3$). Therefore, $\mathcal{A}$ is $\lambda$-system. Now,
let $A\in\Sigma_{1}\times\Sigma_{2}$, then there are $A_{1}\in\Sigma_{1}$
and$A_{2}\in\Sigma_{2}$ such that $A=A_{1}\times A_{2}$ (def $\Sigma_{1}\times\Sigma_{2}$).
Clearly, $\underleftarrow{A}=A_{2}\times A_{1}\in\Sigma_{2}\times\Sigma_{1}$
(def swap and $\Sigma_{2}\times\Sigma_{1}$), i.e., $\Sigma_{1}\times\Sigma_{2}\subset\mathcal{A}$
(def $\mathcal{A}$). Because $\Sigma_{1}\otimes\Sigma_{2}=\sigma\left(\Sigma_{1}\times\Sigma_{2}\right)$
and $\Sigma_{1}\times\Sigma_{2}\subset\mathcal{A}$ then $\underleftarrow{A}\in\Sigma_{2}\otimes\Sigma_{1}$.
\textbf{{[}$\leftarrow${]}} If $\underleftarrow{A}\in\Sigma_{2}\otimes\Sigma_{1}$,
we have that $\underleftarrow{\underleftarrow{A}}\in\Sigma_{1}\otimes\Sigma_{2}$
($\rightarrow$), therefore, $A\in\Sigma_{1}\otimes\Sigma_{2}$ ($A=\underleftarrow{\underleftarrow{A}}$). \end{proof}
\begin{cor}
\label{cor:(commutativity-product-sigma-algebra)} Let $\left(\varOmega_{1},\Sigma_{1}\right)$
and $\left(\varOmega_{2},\Sigma_{2}\right)$ be measurable spaces.
\begin{enumerate}
\item (\textbf{commutativity} \textbf{of $\sigma$-algebra} \textbf{product})
$\Sigma_{1}\otimes\Sigma_{2}\equiv\Sigma_{2}\otimes\Sigma_{1}$.
\item (\textbf{measurability of swap}) The swap function $\underleftarrow{}$
is measurable.
\item (\textbf{swap kernel}) The function $1_{\underleftarrow{}}$ is a
kernel\footnote{We will use the ambiguous notation $1_{\underleftarrow{}}\equiv\underleftarrow{}$
in the rest of this paper.}.
\end{enumerate}
\end{cor}
\begin{proof}
\textbf{{[}1{]}} The swap function is a bijective function. \textbf{{[}2{]}}
Follows from (1) and proposition \ref{prop:swap}.\textbf{ {[}3{]}}
Follows from (2) and theorem \ref{thm:Kernel_det_f}.

Moreover, we can define kernels for deterministic methods that select
a group of individuals from the population (projections).\end{proof}
\begin{lem}
\label{lem:projection-function}(\textbf{projection}) Let $\mathcal{L}=\left\{ \left(\Sigma_{1},\varOmega_{1}\right),\left(\Sigma_{2},\varOmega_{2}\right),\ldots,\left(\Sigma_{n},\varOmega_{n}\right)\right\} $
be a finite ($n\in\mathbb{N}$) ordered list of measurable spaces
and $I=\left\{ k_{1},k_{2},\ldots,k_{m}\right\} \subseteq\left\{ 1,\ldots,n\right\} $
be a set of indices, i.e., $k_{i}<k_{i+1}$ and $m\leq n$. The function
$\pi_{I}$ defined as follows is $\bigotimes_{i=1}^{n}\Sigma_{i}-\bigotimes_{i=1}^{m}\Sigma_{k_{i}}$
measurable.

\[
\begin{array}{rccc}
\pi_{I}: & \prod_{i=1}^{n}\varOmega_{i} & \rightarrow & \prod_{i=1}^{m}\varOmega_{k_{i}}\\
 & \left(x_{1},\ldots,x_{n}\right) & \mapsto & \left(x_{k_{1}},\ldots,x_{k_{n}}\right)
\end{array}
\]

\end{lem}
\begin{proof}
Because $\prod_{i=1}^{n}\varOmega_{i}\equiv\prod_{i=1}^{m}\varOmega_{k_{i}}\times\prod_{i=1}^{n-m}\varOmega_{l_{i}}$
with $I^{c}=\left\{ l_{1},l_{2},\ldots,l_{n-m}\right\} $ the complement
set of indices of $I$ (by applying many times lemma \ref{lem:asso-comm-product}),
we can ``rewrite'' (under equivalences) $\pi_{I}$ as $\pi_{I}\left(x,y\right)=x$
with $y\in\prod_{i=1}^{n-m}\varOmega_{l_{i}}$ and $x\in\prod_{i=1}^{m}\varOmega_{k_{i}}$.
Now, $\bigotimes_{i=1}^{n}\Sigma_{i}\equiv\left(\bigotimes_{i=1}^{m}\Sigma_{k_{i}}\right)\bigotimes\left(\bigotimes_{i=1}^{n-m}\Sigma_{l_{i}}\right)$
(by applying many times proposition \ref{prop:(associativity-of-sigma-algebra}
and corollary \ref{cor:(commutativity-product-sigma-algebra)}). Thus,
for any $A\in\bigotimes_{i=1}^{m}\Sigma_{k_{i}}$ we have that $\pi_{I}^{-1}\left(A\right)=A\times\prod_{i=1}^{n-m}\varOmega_{l_{i}}$.
Clearly, $\pi_{I}^{-1}\left(A\right)\in\left(\bigotimes_{i=1}^{m}\Sigma_{k_{i}}\right)\bigotimes\left(\bigotimes_{i=1}^{n-m}\Sigma_{l_{i}}\right)$
(def. product $\sigma$-algebra), therefore, $\pi_{I}$ is measurable.\end{proof}
\begin{cor}
\label{cor:projection}The function $1_{\pi_{I}}$ as defined in theorem
\ref{thm:Kernel_det_f} is a kernel\footnote{We will use the ambiguous notation $1_{\pi_{I}}\equiv\pi_{I}$ in
the rest of this paper.}.\end{cor}
\begin{proof}
Follows from lemma \ref{lem:projection-function} and theorem \ref{thm:Kernel_det_f}.
\end{proof}
Finally, we are able to define a kernel for a stochastic method that
is the join of several stochastic methods (methods that generate a
subpopulation of the next population). 
\begin{thm}
\label{thm:(Finite-product-probability)}(\textbf{product probability
measure}) Let $\left\{ \left(\varOmega_{i},\Sigma_{i},\mu_{i}\right)\mid i=1,\ldots,n\right\} $
an ordered list of $n\in\mathbb{N}$ probability spaces. There exist
a unique probability measure $\mu:\bigotimes{}_{i=1}^{n}\Sigma_{i}\rightarrow\mathbb{\overline{R}}$
such that $\mu\left(\prod_{i=1}^{n}A_{i}\right)=\prod_{i=1}^{n}\mu\left(A_{i}\right)$
for all $A_{i}\in\Sigma_{i}$, $i=1,2,\ldots,n$. In this case $\mu$
is called the product probability measure of the $\mu_{i}$ probability
measures and is denotated $\bigotimes{}_{i=1}^{n}\mu_{i}$.\end{thm}
\begin{proof}
This theorem is the version of theorem 14.14 in page 277 of the book
of Kenkle \cite{Kenkle14}, when considering $\Sigma_{i}$ not just
a ring but a sigma algebra. In this case, any probability measure
is a finite measure and any finite measure is $\sigma$-finite measure
($\varOmega_{i}\in\Sigma_{i}$).\end{proof}
\begin{thm}
\label{thm:(Join-kernel)}(\textbf{join-kernel}) Let $\left(\varOmega',\Sigma'\right)$
be a measurable space and$\left\{ \left(\varOmega_{i},\Sigma_{i}\right)\right\} $
and $\left\{ K_{i}:\varOmega'\times\Sigma_{i}\rightarrow\left[0,1\right]\right\} $
be ordered lists of $n\in\mathbb{N}$ measurable spaces and kernels,
respectively. The following function is a kernel.

\[
\begin{array}{rccl}
\circledast K\vcentcolon & \varOmega'\times\left(\bigotimes_{i=1}^{n}\Sigma_{i}\right) & \longrightarrow & \left[0,1\right]\\
 & \left(x,A\right) & \longmapsto & \bigotimes_{i=1}^{n}K_{i\left(x,\bullet\right)}\left(A\right)
\end{array}
\]

\end{thm}
\begin{proof}
\textbf{{[}well-defined} \textbf{and }$K.1$\textbf{{]}} Let $x\in\varOmega'$,
since $K_{i\left(x,\bullet\right)}$ is a probability measure for
all $i=1,2,\ldots,n$ ($K.1$ for $K_{i}$ kernel) then $\circledast K_{x,\bullet}=\bigotimes_{i=1}^{n}K_{i\left(x,\bullet\right)}$
is a probability measure $\circledast K_{x,\bullet}\colon\bigotimes_{i=1}^{n}\Sigma_{i}\rightarrow\left[0,1\right]$
(theorem \ref{thm:(Finite-product-probability)}), thus its is well
defined for any $A\in\bigotimes_{i=1}^{n}\Sigma_{i}$. \textbf{{[}$K.2${]}}
Using remark \ref{rem:MeasurableKernelPi}, we just need to prove
that $\circledast K_{\bullet,A}$ is measurable for any $A\in\mathcal{E}$
with $\mathcal{E}\subseteq2^{\prod\varOmega_{i}}$ a $\pi$-system
that generates $\bigotimes_{i=1}^{n}\Sigma_{i}$. Because $\prod_{i=1}^{n}\Sigma_{i}$
is a $\pi$-system (lemma \ref{lem:-sigma-product}) and $\bigotimes_{i=1}^{n}\Sigma_{i}$
is the $\sigma$-algebra generated by $\prod_{i=1}^{n}\Sigma_{i}$
(def. \ref{def:SigmaProduct}), then we just need to prove that $\circledast K_{\bullet,A}$
is measurable for any $A=\prod_{i=1}^{n}A_{i}$ with $A_{i}\in\Sigma_{i}$
for all $i=1,2,\ldots,n$. By definition, $\circledast K\left(x,\prod_{i=1}^{n}A_{i}\right)=\bigotimes_{i=1}^{n}K_{i\left(x,\bullet\right)}\left(\prod_{i=1}^{n}A_{i}\right)$
and according to theorem \ref{thm:(Finite-product-probability)},
$K\left(x,\prod_{i=1}^{n}A_{i}\right)=\prod_{i=1}^{n}K_{i\left(x,\bullet\right)}\left(A_{i}\right)$
($A_{i}\in\Sigma_{i}$). Clearly, $\circledast K_{\left(\bullet,\prod_{i=1}^{n}A_{i}\right)}\left(x\right)=\prod_{i=1}^{n}K_{i\left(\bullet,A_{i}\right)}\left(x\right)$.
Now, $\circledast K_{i\left(\bullet,A_{i}\right)}$ is a measurable
function for all $i=1,2,\ldots,n$ ($K.2$ for $K_{i}$ kernel) then
their product is a measurable function (see Theorem 1.91, page 37
in Kenkle's book \cite{Kenkle14}). Therefore, $\circledast K_{\left(\bullet,\prod_{i=1}^{n}A_{i}\right)}$
is a measurable function.\end{proof}
\begin{cor}
\label{cor:(Join-kernel)}If $\left(\varOmega,\Sigma\right)$ is a
measurable space and $\left\{ n_{i}\in\mathbb{N}^{+}\right\} _{i=1,2,\ldots,m}$
such that $\varOmega_{i}=\varOmega^{n_{i}}$ and $\Sigma_{i}=\Sigma^{\varotimes n_{i}}$
for all $i=1,2,\ldots,m$ in theorem \ref{thm:(Join-kernel)}, then
$\circledast K\vcentcolon\varOmega'\times\Sigma^{\varotimes n}\rightarrow\left[0,1\right]$
, with $n=\sum_{i=1}^{m}n_{i}$, is a kernel.\end{cor}
\begin{prop}
\label{prop:(permutation)}(\textbf{permutation}) Let $\left(\Sigma,\varOmega\right)$
be a measurable space and $I=\left[i_{1},i_{2},\ldots,i_{n}\right]$
be a fixed permutation of the set $\left\{ 1,2,\ldots,n\right\} $
then the function $K_{I}\vcentcolon\varOmega^{n}\times\Sigma^{\varotimes n}\rightarrow\left[0,1\right]$
defined as $K_{I}=\circledast_{k=1}^{n}\pi_{i_{k}}$ is a kernel.\end{prop}
\begin{proof}
Follows from corollaries \ref{cor:(Join-kernel)} and \ref{cor:projection}. \end{proof}
\begin{cor}
\label{cor:permutation} Let $\left(\Sigma,\varOmega\right)$ be a
measurable space and $\mathscr{P}$ be the set of permutations of
set $\left\{ 1,2,\ldots,n\right\} $. Function $K_{\mathscr{P}}\vcentcolon\varOmega^{n}\times\Sigma^{\varotimes n}\rightarrow\left[0,1\right]$
defined as $K_{\mathscr{P}}=\frac{1}{\left|\mathscr{P}\right|}{\displaystyle \sum_{I\in\mathscr{P}}}\pi_{I}$
is a kernel.\end{cor}
\begin{proof}
$K_{\mathscr{P}}$ is a mixing update mechanisms of $\left|\mathscr{P}\right|$
kernels (subsection \ref{sub:Random-Scan-(Mixing)} and proposition
\ref{cor:permutation}).
\end{proof}

\section{Characterization of a\noun{ SGoal} using Probability Theory}

Following the description of a \noun{SGoal} (see Algorithm \ref{Alg:SGoal}),
the initial population $P_{0}$ is chosen according to some initial
distribution $p\left(\cdot\right)$ and the population $P_{t}$ at
step $t>0$ is generated using a stochastic method \noun{(NextPop)}
on the previous population $P_{t-1}$. If such \noun{NextPop} method
can be characterized by a Markov kernel, the stochastic sequence $\left(P_{t}\vcentcolon t\geq0\right)$
becomes a Markov chain. In order to develop this characterization,
first we define appropiated measurable spaces and Markov kernels of
stochastic methods, and then we define some properties of stochastic
methods that cover many popular \noun{SGoal}s reported in the literature.

Since a \noun{SGoal} consists of a population of $n$ individuals
on the feasible region $\varOmega$, it is clear that the state space
is defined on $\varOmega^{n}$. Moreover, the initial population $P_{0}\in\varOmega^{n}$
is chosen according to some initial distribution $p\left(\cdot\right)$.
Now, the $\sigma$-algebra must allow us to determine convergence
properties on the kernel. In this paper, we will extend the convergence
approach proposed by Günter Rudolph in \cite{Rudolph96convergenceof}
to \noun{SGoal}s. In the following, we call \textbf{\emph{objective
function}} to a function $f\vcentcolon\varPhi\rightarrow\mathbb{R}$
if its has an optimal value (denoted as $f^{*}\in\mathbb{R}$) in
the feasible region.

\subsection{$\epsilon$-optimal states}

We define and study the set of strict $\epsilon$-optimal states (the
optimal elements according to Rudolph's notation), i.e., a set that
includes any candidate population which best individual has a value
of the objective function close (less than $\epsilon\in\mathbb{R}^{+}$)
to the optimum objective function value. We also introduce two new
natural definitions that we will use in some proofs (the $\epsilon$-optimal
states and $\epsilon$-states) and study some properties of sets defined
upon these concepts.
\begin{defn}
\label{def:(epsilon-optimal-elements)}Let $\varOmega\subseteq\varPhi$
be a set, $f\vcentcolon\varPhi\rightarrow\mathbb{R}$ be an objective
function, $\epsilon>0$ be a real number and $x\in\varOmega^{m}$.
\begin{enumerate}
\item (\textbf{optimality}) $d\left(x\right)=f\left(\mbox{\textsc{Best}}\left(x\right)\right)-f^{*}$
(here $f^{*}$ is the optimal value of $f$ in $\varOmega$).
\item (\textbf{strict }$\epsilon$\textbf{-optimum state}) $x$ is an strict
$\epsilon$-optimum element if $d\left(x\right)<\epsilon$, 
\item ($\epsilon$\textbf{-optimum state}) $x$ is an $\epsilon$-optimum
element if $d\left(x\right)\leq\epsilon$, and
\item ($\epsilon$\textbf{-state}) $x$ is an $\epsilon$-element if $d\left(x\right)=\epsilon$.
\end{enumerate}
\end{defn}
Sets $\varOmega_{\epsilon}^{m}=\left\{ x\in\varOmega^{m}\vcentcolon d\left(x\right)<\epsilon\right\} $,
$\varOmega_{\overline{\epsilon}}^{m}=\left\{ x\in\varOmega^{m}\vcentcolon d\left(x\right)\leq\epsilon\right\} $,
and $\varOmega_{\mathring{\epsilon}}^{m}=\left\{ x\in\varOmega^{m}\vcentcolon d\left(x\right)=\epsilon\right\} $
are called set of strict $\epsilon$-optimal states, $\epsilon$-optimal
states, and $\epsilon$-states, respectively. We will denotate $\varOmega_{\epsilon}=\varOmega_{\epsilon}^{1}$,
$\varOmega_{\overline{\epsilon}}=\varOmega_{\overline{\epsilon}}^{1}$
and $\varOmega_{\mathring{\epsilon}}=\varOmega_{\mathring{\epsilon}}^{1}$
.
\begin{rem}
\label{rem:complements-optimal-sets}Notice that
\begin{enumerate}
\item $\left(\varOmega_{\epsilon}^{m}\right)^{c}=\left\{ x\in\varOmega^{m}\vcentcolon\epsilon\leq d\left(x\right)\right\} $
and
\item $\left(\varOmega_{\overline{\epsilon}}^{m}\right)^{c}=\left\{ x\in\varOmega^{m}\vcentcolon\epsilon<d\left(x\right)\right\} $.
\end{enumerate}
\end{rem}
\begin{lem}
\label{lem:optimal-states}$\varOmega_{\overline{\epsilon}}^{m}={\displaystyle \bigcap_{n=1}^{\infty}}\varOmega_{\epsilon+\frac{1}{n}}^{m}$
for all $\epsilon>0$ and $m\in\mathbb{N}^{+}$.\end{lem}
\begin{proof}
\textbf{{[}$\subseteq${]}} If $x\in\varOmega_{\overline{\epsilon}}^{m}$
then $d\left(x\right)\leq\epsilon$ (def $\varOmega_{\overline{\epsilon}}^{m}$).
Now, $\epsilon<\epsilon+\frac{1}{n}$ for all $n>0$, clearly $d\left(x\right)<\epsilon+\frac{1}{n}$
for all $n>0$. Therefore, $x\in\varOmega_{\epsilon+\frac{1}{n}}^{m}$
for all $n>0$ so $x\in{\displaystyle \bigcap_{n=1}^{\infty}}\varOmega_{\epsilon+\frac{1}{n}}^{m}$.
\textbf{{[}}$\supseteq$\textbf{{]}} Let $x\in{\displaystyle \bigcap_{n=1}^{\infty}}\varOmega_{\epsilon+\frac{1}{n}}^{m}$,
if $x\notin\varOmega_{\overline{\epsilon}}^{m}$ then $d\left(x\right)>\epsilon$
(def $\varOmega_{\overline{\epsilon}}^{m}$), therefore, $\exists\delta>0$
such that $d\left(x\right)=\epsilon+\delta$. We have that $\exists n\in\mathbb{N}$
such that $0<\frac{1}{n}<\delta$ (Archims theorem and real numbers
dense theorem). Clearly, $\epsilon+\frac{1}{n}<\epsilon+\delta=d\left(x\right)$,
so $x\notin\varOmega_{\epsilon+\frac{1}{n}}^{m}$, then $x\notin{\displaystyle \bigcap_{n=1}^{\infty}}\varOmega_{\epsilon+\frac{1}{n}}^{m}$
(contradiction). Therefore, $x\in\varOmega_{\overline{\epsilon}}^{m}$.
\end{proof}

\subsection{Optimization space}

We define the optimization $\sigma$-algebra property (a $\sigma$-algebra
containing the family of sets of strict $\epsilon$-optimal states)
and show that such property is preserved by the product $\sigma$-algebra.
\begin{defn}
\textbf{($f-$optimization $\sigma$-algebra)} Let $f\vcentcolon\varPhi\rightarrow\mathbb{R}$
be an objective function, and $\varOmega\subseteq\varPhi$. A $\sigma$-algebra
$\Sigma$ on $\varOmega$ is called $f-$optimization $\sigma$-algebra
iff $\left\{ \varOmega_{\epsilon}\right\} _{\epsilon>0}\subseteq\Sigma$.\end{defn}
\begin{lem}
\label{lem:optimal-sets-sigma-algebra}Let $\Sigma$ be an $f$-optimization
$\sigma$-algebra on $\varOmega$ then $\left\{ \varOmega_{\overline{\epsilon}}\right\} _{\epsilon>0}\subseteq\Sigma$
and $\left\{ \varOmega_{\mathring{\epsilon}}\right\} _{\epsilon>0}\subseteq\Sigma$
.\end{lem}
\begin{proof}
$\left[\left\{ \varOmega_{\overline{\epsilon}}\right\} _{\epsilon>0}\subseteq\Sigma\right]$
It follows from the facts that $\left\{ \varOmega_{\epsilon+\frac{1}{n}}\right\} _{n\in\mathbb{N}^{+}}\subseteq\left\{ \varOmega_{\epsilon}\right\} \subseteq\Sigma$
($\Sigma$ optimization $\sigma$-algebra), $\varOmega_{\overline{\epsilon}}={\displaystyle \bigcap_{n=1}^{\infty}}\varOmega_{\epsilon+\frac{1}{n}}$
(lemma \ref{lem:optimal-states}) and $\Sigma$ close under countable
intersections (part 2, lemma \ref{lem:sigma-intersections}). $\left[\left\{ \varOmega_{\mathring{\epsilon}}\right\} _{\epsilon>0}\subseteq\Sigma\right]$
It follows from the fact that $\varOmega_{\mathring{\epsilon}}=\varOmega_{\overline{\epsilon}}\setminus\varOmega_{\epsilon}$
for all $\epsilon>0$ and $\Sigma$ is close under proper differences
(part 4, lemma \ref{lem:sigma-intersections}).\end{proof}
\begin{prop}
$\varOmega_{\epsilon}^{m}={\displaystyle \bigcup_{i=1}^{m}}\left[\varOmega^{i-1}\times\varOmega_{\epsilon}\times\varOmega^{m-i}\right]$
for all $\epsilon>0$ and $m\in\mathbb{N}^{+}$.\end{prop}
\begin{proof}
\textbf{{[}$\subseteq${]}} Let $x\in\varOmega_{\epsilon}^{m}$, then
$d\left(x\right)<\epsilon$ (def $\varOmega_{\epsilon}^{m}$) and
$f\left(\mbox{\textsc{Best}}\left(x\right)\right)-f^{*}<\epsilon$
(def $d\left(x\right)$). It is clear that $x={\displaystyle \prod_{k=1}^{i-1}}x_{k}\times x_{i}\times{\displaystyle \prod_{k=i+1}^{m}}x_{k}$
$f\left(x_{i}\right)-f^{*}<\epsilon$ for some $i=1,2,\ldots,m$ (def
\noun{Best}), so $d\left(x_{i}\right)<\epsilon$ (def $d\left(x\right)$).
Therefore, $x_{i}\in\varOmega_{\epsilon}$ (def $\varOmega_{\epsilon}$)
so $x\in\varOmega^{i-1}\times\varOmega_{\epsilon}\times\varOmega^{m-i}$
and $x\in{\displaystyle \bigcup_{i=1}^{m}}\left[\varOmega^{i-1}\times\varOmega_{\epsilon}\times\varOmega^{m-i}\right]$.
\textbf{{[}$\supseteq${]}} if $x\in{\displaystyle \bigcup_{i=1}^{m}}\left[\varOmega^{i-1}\times\varOmega_{\epsilon}\times\varOmega^{m-i}\right]$
then $x\in\left[\varOmega^{i-1}\times\varOmega_{\epsilon}\times\varOmega^{m-i}\right]$
for some $i=1,2,\ldots,m$. Clearly, $f\left(\mbox{\textsc{Best}}\left(x\right)\right)<f\left(x_{i}\right)$
(def \noun{Best}) so $d\left(x\right)\leq d\left(x_{i}\right)$. Now,
$d\left(x\right)\leq d\left(x_{i}\right)<\epsilon$ ($x_{i}\in\varOmega_{\epsilon}$)
therefore $x\in\varOmega_{\epsilon}^{m}$ (def $\varOmega_{\epsilon}^{m}$
).\end{proof}
\begin{cor}
Let $\Sigma$ be an $f$-optimization $\sigma$-algebra on $\varOmega$
then $\Sigma^{\varotimes m}$ is an $f-$optimization $\sigma$-algebra
on $\varOmega^{m}$ for all $m\in\mathbb{N}^{+}$. \end{cor}
\begin{proof}
$\varOmega_{\epsilon}\in\Sigma$ (optimization $\sigma$-algebra)
and $\varOmega^{i}\in\Sigma^{\otimes i}$ for all $i=1,2,\ldots,m$
(universality of $\sigma$-algebra) then $\left[\varOmega^{i-1}\times\varOmega_{\epsilon}\times\varOmega^{m-i}\right]\in\Sigma^{\otimes m}$
(def product $\sigma$-algebra), so ${\displaystyle \bigcup_{i=1}^{m}}\left[\varOmega^{i-1}\times\varOmega_{\epsilon}\times\varOmega^{m-i}\right]\in\Sigma^{\otimes m}$
($\Sigma^{\otimes m}$ is \noun{$\overline{\mbox{\textsc{cdu}}}$}).
Therefore, $\varOmega_{\epsilon}^{m}\in\Sigma^{\otimes m}$ for all
$\epsilon>0$, i.e., $\Sigma^{\otimes m}$ is an optimization $\sigma$-algebra.
\end{proof}
Now, we are ready to define the mathematical structure that we use
for characterizing a \noun{SGoal}.
\begin{defn}
\label{def:(optimization-space)}(\textbf{optimization space}) If
$f\vcentcolon\varPhi\rightarrow\mathbb{R}$ is an objective function,
$n\in\mathbb{N}$ and $\Sigma$ is an $f-$optimization $\sigma$-algebra
over a set $\varOmega$ then the triple $\left(\varOmega^{n},\Sigma^{\otimes n},f\right)$
is called optimization space.
\end{defn}

\subsection{Kernels on optimization spaces}

If we are provided with an optimization space $\left(\varOmega^{n},\Sigma^{\otimes n},f\right)$,
we can represent the population used by the \noun{SGoal}, as an individual
$x\in\varOmega^{n}$, while we can characterize the \noun{NextPop}
(a $\mbox{\textsc{f}}\vcentcolon\Omega^{n}\rightarrow\varOmega^{n}$
stochastic method), as a Markov kernel $K\vcentcolon\varOmega^{n}\times\Sigma^{\otimes n}\rightarrow\left[0,1\right]$.
Because such \noun{NextPop }can be defined in terms of more general
$\mbox{\textsc{f}}\vcentcolon\Omega^{\eta}\rightarrow\varOmega^{\upsilon}$
stochastic methods, we study such general kernels.

\subsubsection{\label{sub:Join-Stochastic-Methods}Join Stochastic Methods}
\begin{defn}
\label{def:(Join)-SM}(\textbf{Join}) A stochastic method $\mbox{\textsc{f}}\vcentcolon\Omega^{\eta}\rightarrow\varOmega^{\upsilon}$
is called join if it is defined as the join of $m\in\mathbb{N}$ stochastic
methods ($\mbox{\textsc{f}}_{i}\vcentcolon\varOmega^{n}\rightarrow\varOmega^{\upsilon_{i}}$),
each method generating a subpopulation of the population, i.e. $\mbox{\textsc{f}}={\displaystyle \prod_{i=1}^{m}\mbox{\textsc{f}}_{i}}$,
(see Algorithm \ref{Alg:JoinedMethod}). Here, $\upsilon_{i}\in\mathbb{N}^{+}$
is the size of the $i$-th sub-population $i=1,2,\ldots,m$, and $s_{m}=\upsilon$.

\begin{algorithm}[htbp]
\noun{f}($P$)

1.~~$s_{i}=\sum_{k=1}^{i-1}\upsilon_{k}$ for all $i=1,2,\ldots,m$

2.~~$Q_{1+s_{i-1},\ldots,s_{i}}=$\noun{f}$_{i}$($P$) for all
$i=1,2,\ldots,m$

3.~~\textbf{return} $Q$

\caption{\label{Alg:JoinedMethod}Joined Stochastic Method. }
\end{algorithm}

\end{defn}
\begin{example}
\label{exa:NextPopGGA}The \noun{NextPop} method of a \noun{GGa} (see
Algorithm \ref{Alg:GGA}), is a joined stochastic method: a stochastic
method $\mbox{\textsc{NextSubPop}}{}_{\textsc{GGa}}\vcentcolon\varOmega^{n}\rightarrow\varOmega^{2}$
that generates two new candidate solutions (by selecting two parents
from the population, recombining them and mutating the offspring),
is applied $\frac{n}{2}$ times\noun{ }in order to generate the next
population, see equation \ref{eq:JoinedGGA}.
\end{example}
\begin{equation}
\mbox{\textsc{NextPop}}{}_{\textsc{GGa}}\left(P\right)={\displaystyle \prod_{i=1}^{\frac{n}{2}}}\mbox{\textsc{NextSubPop}}{}_{\textsc{GGa}}\left(P\right)\label{eq:JoinedGGA}
\end{equation}

\begin{example}
\label{exa:NextPopDE}The\noun{ NextPop} method of a \noun{DE} (see
Algorithm \ref{Alg:DE}), is a joined stochastic method: $n$ stochastic
methods $\mbox{\textsc{NextInd}}{}_{\textsc{DE},i}\vcentcolon\varOmega^{n}\rightarrow\varOmega$
each one generating the $i$th individual of the new population (by
selecting three extra parents from the population and recombining
each dimension using differences if required), are applied, see equation
\ref{eq:JoinedGGA}.
\end{example}
\begin{equation}
\mbox{\textsc{NextPop}}{}_{\textsc{DE}}\left(P\right)={\displaystyle \prod_{i=1}^{n}}\mbox{\textsc{NextInd}}{}_{\textsc{DE},i}\left(P\right)\label{eq:JoinedDE}
\end{equation}

\begin{example}
\label{exa:NextPopPHC}The\noun{ NextPop} method of a \noun{PHC} (see
Algorithm \ref{Alg:Parallel-Hill-Climbing}), is a joined stochastic
method: $n$ stochastic methods $\mbox{\textsc{NextInd}}{}_{\textsc{PHC},i}\vcentcolon\varOmega^{n}\rightarrow\varOmega$
each one generating the $i$th individual of the new population ($\mbox{\textsc{NextInd}}{}_{\textsc{PHC},i}\left(P\right)=\mbox{\textsc{NextPop}}{}_{\textsc{HC}}\left(P_{i}\right)$),
see equation \ref{eq:JoinedPHC}.
\end{example}
\begin{equation}
\mbox{\textsc{NextPop}}{}_{\textsc{DE}}\left(P\right)={\displaystyle \prod_{i=1}^{n}}\mbox{\textsc{NextPop}}{}_{\textsc{HC}}\left(P_{i}\right)\label{eq:JoinedPHC}
\end{equation}

We are now in the position of providing a sufficient condition for
characterizing \noun{PHC}, \noun{GGa}, and \noun{DE} algorithms.
\begin{prop}
\label{prop:Join-stochastic-kernel}Let $\left\{ \mbox{\textsc{f}}_{i}\vcentcolon\varOmega^{n}\rightarrow\varOmega^{\upsilon_{i}}\right\} _{i=1,2,\ldots,m}$
be a finite family of stochastic methods, each one characterized by
a kernel $K_{i}\vcentcolon\varOmega^{\eta}\times\Sigma^{\otimes\upsilon_{i}}\rightarrow\left[0,1\right]$,
then the join stochastic method $\mbox{\textsc{f}}={\displaystyle \prod_{i=1}^{m}\mbox{\textsc{f}}_{i}}$
is characterized by the kernel $\circledast K\vcentcolon\varOmega^{\eta}\times\Sigma^{\otimes\upsilon}\rightarrow\left[0,1\right]$
with $\upsilon=\sum_{k=1}^{m}\upsilon_{k}$.\end{prop}
\begin{proof}
Follows from corollary \ref{cor:(Join-kernel)} of theorem \ref{thm:(Join-kernel)}.\end{proof}
\begin{cor}
\label{cor:Join-stoch}Each of the \noun{NextPop$_{\textsc{PHC}}$,
NextPop$_{\textsc{GGa}}$} and \noun{NextPop$_{\textsc{DE}}$} stochastic
methods can be characterized by kernels if\textup{\emph{ each of the}}\textup{
}stochastic methods\textup{\noun{ NextPop}}$_{\textsc{HC}}$,\textup{\noun{
NextSubPop}}$_{\textsc{GGa}}$, and \textup{\noun{Next}}\noun{Ind}$_{\textsc{DE}}$
can be characterized by a kernel.
\end{cor}

\subsubsection{Sorting Methods}

Although the result of sorting a population is, in general, a non
stochastic method, we can model it as a kernel. We start by modeling
the sorting of two elements according to their fitness value.
\begin{defn}
\label{def:sorting-two-method}(\textbf{\noun{Sort-Two}}) Let $d\vcentcolon\varOmega\rightarrow\mathbb{R}$,
the sort-two function $\mbox{s}_{2}\vcentcolon\varOmega^{2}\rightarrow\varOmega^{2}$
is defined as follows:

\[
\mbox{\textsc{s}}_{2}\left(z=\left(x,y\right)\right)=\left\{ \begin{array}{ll}
z & \mbox{if }d\left(x\right)<d\left(y\right)\\
\underleftarrow{z}=\left(y,x\right) & \mbox{otherwise}
\end{array}\right.
\]

\end{defn}
In order to model the S$_{2}$ method as a kernel, we need to define
sets that capture some notions of sorted couples.
\begin{defn}
(\textbf{sorted couples sets}) Let $x,y\in\varOmega$.
\begin{enumerate}
\item[\noun{$\mbox{\textsc{mM}}$}]  The set $\mbox{\textsc{mM}}=\left\{ \left(x,y\right)\mid d\left(x\right)<d\left(y\right)\right\} $
is called min-max sorted couples set.
\item[\noun{$\mbox{\textsc{Mm}}$}]  The set $\mbox{\textsc{Mm}}=\left\{ \left(x,y\right)\mid d\left(y\right)<d\left(x\right)\right\} $
is called max-min sorted couples set.
\item[\noun{$\mbox{\textsc{m}}$}]  The set $\mbox{\textsc{m}}=\left\{ \left(x,y\right)\mid d\left(y\right)=d\left(x\right)\right\} $
is called equivalent couples set.
\end{enumerate}
\end{defn}
\begin{lem}
\label{lem:Best2Set}The following set of equations holds.
\begin{enumerate}
\item $\mbox{\textsc{mM}}={\displaystyle \bigcup_{r\in\mathbb{Q}}\varOmega_{r}\times\varOmega_{\overline{r}}^{c}}$
\item $\mbox{\textsc{Mm}}={\displaystyle \bigcup_{r\in\mathbb{Q}}\varOmega_{\overline{r}}^{c}}\times\varOmega_{r}$
\item $\mbox{\textsc{m}}=\left(\mbox{\textsc{mM}}\biguplus\mbox{\textsc{Mm}}\right)^{c}$
\end{enumerate}
\end{lem}
\begin{proof}
\textbf{{[}1{]}} $\left[\subseteq\right]$ Let $\left(x,y\right)\in\mbox{\textsc{mM}}$
then $d\left(x\right)<d\left(y\right)$ (def. $\mbox{\textsc{mM}}$),
therefore $\exists r\in\mathbb{Q}$ such that $d\left(x\right)<r<d\left(y\right)$
(Archims theorem and real numbers dense theorem). Clearly, $x\in\varOmega_{r}$
and $y\in\varOmega_{\overline{r}}^{c}$ (def $\varOmega_{\epsilon}$,$\varOmega_{\overline{\epsilon}}$
and remark \ref{rem:complements-optimal-sets}), then $\left(x,y\right)\in\varOmega_{r}\times\varOmega_{\overline{r}}^{c}$,
so $\left(x,y\right)\in{\displaystyle \bigcup_{r\in\mathbb{Q}}\varOmega_{r}\times\varOmega_{\overline{r}}^{c}}$
. $\left[\supseteq\right]$ Let $\left(x,y\right)\in{\displaystyle \bigcup_{r\in\mathbb{Q}}\varOmega_{r}\times\varOmega_{\overline{r}}^{c}}$
then $\exists r\in\mathbb{Q}$ such that $\left(x,y\right)\in\varOmega_{r}\times\varOmega_{\overline{r}}^{c}$,
therefore $x\in\varOmega_{r}$ and $y\in\varOmega_{\overline{r}}^{c}$.
Clearly, $d\left(x\right)<r$ and $r<d\left(y\right)$ (def $\varOmega_{\epsilon}$,$\varOmega_{\overline{\epsilon}}$
and remark \ref{rem:complements-optimal-sets}), then $d\left(x\right)<d\left(y\right)$
so $\left(x,y\right)\in\mbox{\textsc{mM}}$ (def $\mbox{\textsc{mM}}$).
\textbf{{[}2{]}} It is a proof similar to the proof of part {[}1{]}.
\textbf{{[}3{]}} Obvious, for any $\left(x,y\right)\in{\displaystyle \varOmega}\times\varOmega$
we have that $d\left(x\right)<d\left(y\right)\vee d\left(y\right)<d\left(x\right)\vee d\left(x\right)=d\left(y\right)$
($\mathbb{R}$ is total order). Clearly, $\mbox{\textsc{mM, }}\mbox{\textsc{Mm}, and }\mbox{\textsc{m}}$
are pairwise disjoint sets. Then $\mbox{\ensuremath{\left({\displaystyle \varOmega}\times\varOmega\right)}=\textsc{mM}}\biguplus\mbox{\textsc{Mm}}\biguplus\mbox{\textsc{m}}$,
so $\mbox{\textsc{m}}=\left({\displaystyle \varOmega}\times\varOmega\right)\setminus\left(\mbox{\textsc{mM}}\biguplus\mbox{\textsc{Mm}}\right)=\left(\mbox{\textsc{mM}}\biguplus\mbox{\textsc{Mm}}\right)^{c}$.\end{proof}
\begin{lem}
\label{lem:mM-Mm-m}$\mbox{\textsc{mM, }}\mbox{\textsc{Mm}, }\mbox{\textsc{m}, }\mbox{\textsc{mM\ensuremath{^{c}}, }}\mbox{\textsc{Mm}\ensuremath{\ensuremath{^{c}}}, }\mbox{\textsc{m\ensuremath{\ensuremath{^{c}}}}}\in\Sigma^{\varotimes2}$
if $\Sigma$ is an optimization $\sigma$-algebra. \end{lem}
\begin{proof}
$\left(\varOmega_{r}\times\varOmega_{\overline{r}}^{c}\right),\left(\varOmega_{\overline{r}}^{c}\times\varOmega_{r}\right)\in\Sigma^{\varotimes2}$
($\varOmega_{r},\varOmega_{\overline{r}}^{c}\in\Sigma$, def product
$\sigma$-algebra), and $\mbox{\textsc{mM}}{\displaystyle =\bigcup_{r\in\mathbb{Q}}\varOmega_{r}\times\varOmega_{\overline{r}}^{c}}$
and $\mbox{\textsc{Mm}}{\displaystyle =\bigcup_{r\in\mathbb{Q}}\varOmega_{\overline{r}}^{c}}\times\varOmega_{r}$
(lemma \ref{lem:Best2Set}), then $\mbox{\textsc{mM, }}\mbox{\textsc{Mm}}\in\Sigma^{\varotimes2}$
($\Sigma^{\varotimes2}$ is \noun{$\overline{\mbox{\textsc{cdu}}}$}).
Clearly, $\left(\mbox{\textsc{mM}}\biguplus\mbox{\textsc{Mm}}\right)\mbox{\textsc{m}}\in\Sigma^{\varotimes2}$
($\Sigma^{\varotimes2}$ is \noun{$\overline{\mbox{\textsc{cdu}}}$}
), so $\mbox{\textsc{m}}\in\Sigma^{\varotimes2}$ ($\Sigma^{\varotimes2}$
is \noun{$\overline{\mbox{\textsc{c}}}$). }Finally, $\mbox{\textsc{mM\ensuremath{^{c}}, }}\mbox{\textsc{Mm}\ensuremath{\ensuremath{^{c}}}, }\mbox{\textsc{m\ensuremath{\ensuremath{^{c}}}}}\in\Sigma^{\varotimes2}$
($\sigma.2$).\end{proof}
\begin{prop}
\label{prop:Sort-measurable}$\mbox{s}_{2}\vcentcolon\varOmega^{2}\rightarrow\varOmega^{2}$
is measurable.\end{prop}
\begin{proof}
Let $A\in\Sigma$ and $z=\left(x,y\right)\in\varOmega^{2}$. $z\in\mbox{\textsc{s}}_{2}^{-1}\left(A\right)$
iff $\left[z\in A\wedge d\left(x\right)<d\left(y\right)\right]\vee\left[\underleftarrow{z}\in A\wedge d\left(y\right)\leq d\left(x\right)\right]$
(def $\mbox{\textsc{s}}_{2}$) iff $\left[z\in A\wedge z\in\mbox{\textsc{mM}}\right]\vee\left[\underleftarrow{z}\in A\wedge z\in\mbox{\textsc{mM}}^{c}\right]$
(def $\mbox{\textsc{mM}}$) iff $z\in\left(A\bigcap\mbox{\textsc{mM}}\right)\biguplus\left(\underleftarrow{A}\bigcap\mbox{\textsc{mM}}^{c}\right)$
(def $\bigcup$ and $\bigcap$). Since $A,\varOmega\in\Sigma$ then
$A,\underleftarrow{A},\mbox{\textsc{mM}},\mbox{\textsc{mM}}^{c}\in\Sigma^{\otimes2}$
(corollary \ref{prop:swap} and lemma \ref{lem:mM-Mm-m}). Therefore,
$\mbox{\textsc{s}}_{2}^{-1}\left(A\right)\in\Sigma^{\otimes2}$ ($\Sigma^{\otimes2}$
is \noun{$\overline{\mbox{\textsc{c}}}$}, \noun{$\overline{\mbox{\textsc{cu}}}$,}
and \noun{$\overline{\mbox{\textsc{ci}}}$}). Clearly, $\mbox{\textsc{s}}_{2}$
is measurable.\end{proof}
\begin{cor}
\label{cor:Sort-kernel}$1_{\mbox{\textsc{s}}_{2}}:\varOmega^{2}\times\Sigma^{\varotimes2}\rightarrow\left[0,1\right]$
as defined in theorem \ref{thm:Kernel_det_f} is a kernel. \end{cor}
\begin{proof}
Follows from proposition \ref{prop:Sort-measurable} and theorem \ref{thm:Kernel_det_f}.
\end{proof}
Having defined the kernel for \noun{s$_{2}$}, we define a kernel
$\mbox{\textsc{s}}_{n,n-1}\vcentcolon\varOmega^{n}\times\Sigma^{\varotimes n}\rightarrow\left[0,1\right]$
for characterizing a $n$-tuple sorting method.
\begin{prop}
The following functions are kernels
\begin{enumerate}
\item $\mbox{\textsc{w}}_{n,k}\vcentcolon\varOmega^{n}\times\Sigma^{\varotimes n}\rightarrow\left[0,1\right]$
defined as $\mbox{\textsc{w}}_{n,k}=\pi_{\left\{ 1,\ldots,k-1\right\} }\circledast\left[\mbox{\textsc{s}}_{2}\circ\pi_{\left\{ k,k+1\right\} }\right]\circledast\pi_{\left\{ k+2,\ldots,n\right\} }$
for $k=1,\ldots,n-1$.
\item $\mbox{\textsc{t}}_{n,k}\vcentcolon\varOmega^{n}\times\Sigma^{\varotimes n}\rightarrow\left[0,1\right]$
defined as $\mbox{\textsc{t}}_{n,1}=\mbox{\textsc{w}}_{n,1}$, and
$\mbox{\textsc{t}}_{n,k}=\mbox{\textsc{w}}_{n,k}\circ\mbox{\textsc{t}}_{n,k-1}$
for $k=2,\ldots,n-1$.
\item $\mbox{\textsc{s}}_{n,k}\vcentcolon\varOmega^{n}\times\Sigma^{\varotimes n}\rightarrow\left[0,1\right]$
defined as $\mbox{s}_{n,1}=\mbox{\textsc{t}}_{n,1}$, and $\mbox{\textsc{s}}_{n,k}=\mbox{\textsc{t}}_{n,k}\circ\mbox{\textsc{s}}_{n,k-1}$
for $k=2,\ldots,n-1$.
\end{enumerate}
\end{prop}
\begin{proof}
Obvious, all functions are defined in terms of composition and/or
join of kernels. \end{proof}
\begin{cor}
The \noun{Best$_{2}$} function used by the \noun{SSGa} (line 5, algorihm
\ref{Alg:SSGA}) can be characterized by the kernel $\mbox{\textsc{b}}_{2,4}\vcentcolon\varOmega^{4}\times\Sigma^{\otimes2}\rightarrow\left[0,1\right]$
defined as $\mbox{\textsc{b}}_{2,4}=\pi_{\left\{ 1,2\right\} }\circ\mbox{\textsc{s}}_{n,2}$.\end{cor}
\begin{prop}
\label{prop:NextPop-SSGa}If lines 3-4 in the \noun{SSGa} (see algorithm
\ref{Alg:SSGA}) can be modeled by a kernel $\mbox{\textsc{v}}\vcentcolon\varOmega^{2}\times\Sigma^{\otimes2}\rightarrow\left[0,1\right]$,
The stochastic method \noun{NextPop$_{\mbox{\textsc{SSGa}}}$} can
be characterized by the following kernel.

$K_{\mbox{\textsc{SSGa}}}=\left[\left[\mbox{\textsc{b}}_{2,4}\circ\pi_{\left\{ 1,\ldots,4\right\} }\right]\varoast\pi_{\left\{ 5,\ldots,n+2\right\} }\right]\circ\left[\left[\mbox{\textsc{v}}\circ\pi_{\left\{ 1,2\right\} }\right]\varoast1\right]\circ K_{\mathscr{P}}$

\end{prop}

\subsubsection{\label{sub:Variation-Replacement-Stochastic}Variation-Replacement
Stochastic Methods}

Many \noun{SGoal}s are defined as two-steps stochastic processes:\emph{
}First by applying a stochastic method that generates $\varpi\in\mathbb{N}$
new individuals, in order to \emph{``explore''} the search space,
and then by applying a stochastic method that selects candidate solutions
among the current individuals and the new individuals, in order to
\emph{``improve''} the quality of candidate solutions.
\begin{defn}
\label{def:(Variation-Replacement)}(\textbf{Variation-Replacement})
A stochastic method $\mbox{\textsc{f}}\vcentcolon\varOmega^{\eta}\rightarrow\varOmega^{\upsilon}$
is called Variation-Replacement \noun{(}\textbf{\noun{VR}}\noun{)
}if there are two stochastic methods, $\mbox{\textsc{v}}\vcentcolon\varOmega^{\eta}\rightarrow\varOmega^{\varpi}$
and $\mbox{\textsc{r}}\vcentcolon\varOmega^{\eta+\varpi}\rightarrow\varOmega^{\upsilon}$,
(t) such that $\mbox{\textsc{f}}\left(P\right)=\mbox{\textsc{r}}\left(P,\mbox{\textsc{v}}\left(P\right)\right)$
or $\mbox{\textsc{f}}\left(P\right)=\mbox{\textsc{r}}\left(\mbox{\textsc{v}}\left(P\right),P\right)$
for all $P\in\varOmega^{\eta}$.\end{defn}
\begin{example}
\label{exa:NextPopHC}The \noun{NextPop} method of HC with neutral
mutations (see Algorithm \ref{Alg:Hill-Climbing}) is a VR stochastic
method, see equations \ref{eq:HC-VR} and \ref{eq:R-HC}. The HC algorithm
will not consider neutral mutations just by changing the order of
the arguments in the replacement stochastic method R$_{\mbox{\textsc{HC}}}$,
i.e., $\mbox{\textsc{R}}{}_{\textsc{HC}}\left(\textsc{Variate}\left(x\right),x\right)$.
\end{example}
\begin{equation}
\mbox{\textsc{NextPop}}{}_{\textsc{HC}}\left(x\right)=\mbox{\textsc{R}}{}_{\textsc{HC}}\left(x\mbox{,}\mbox{\textsc{Variate}}\left(x\right)\right)\label{eq:HC-VR}
\end{equation}

\begin{equation}
\mbox{\textsc{r}}{}_{\textsc{HC}}\left(x,y\right)=\left\{ \begin{array}{ll}
x & \mbox{if }f\left(x\right)<f\left(y\right)\\
y & \mbox{otherwise}
\end{array}\right.\label{eq:R-HC}
\end{equation}

\begin{prop}
\label{prop:VR-Kernel}If $\mbox{\textsc{v}}\vcentcolon\varOmega^{\eta}\rightarrow\varOmega^{\varpi}$
and $\mbox{\textsc{r}}\vcentcolon\varOmega^{\eta+\varpi}\rightarrow\varOmega^{\upsilon}$
are stochastic methods characterized by kernels $K_{\mbox{\textsc{v}}}\vcentcolon\varOmega^{\eta}\times\Sigma^{\otimes\varpi}\rightarrow\left[0,1\right]$
and $K_{\mbox{\textsc{r}}}\vcentcolon\varOmega^{\eta+\varpi}\times\Sigma^{\otimes\upsilon}\rightarrow\left[0,1\right]$,
respectively, then $K_{\mbox{\textsc{f}}}=K_{\mbox{\textsc{r}}}\circ\left[1_{\varOmega^{\eta}}\varoast K_{\mbox{\textsc{v}}}\right]$
and $K_{\mbox{\textsc{f}}}=K_{\mbox{\textsc{r}}}\circ\left[K_{\mbox{\textsc{v}}}\varoast1_{\varOmega^{\eta}}\right]$
are kernels that characterize the VR stochastic method $\mbox{\textsc{f}}\left(P\right)=\mbox{\textsc{r}}\left(P,\mbox{\textsc{v}}\left(P\right)\right)$
and $\mbox{\textsc{f}}\left(P\right)=\mbox{\textsc{r}}\left(\mbox{\textsc{v}}\left(P\right),P\right)$
with $P\in\varOmega^{\eta}$, respectively.\end{prop}
\begin{proof}
Clearly, $\left[1_{\varOmega^{\eta}}\varoast K_{\mbox{\textsc{v}}}\right]$
and $\left[K_{\mbox{\textsc{v}}}\varoast1_{\varOmega^{\eta}}\right]$
are kernels (theorem \ref{thm:(Join-kernel)} and lemma \ref{thm:Kernel_det_f}).
Therefore, $K_{\mbox{\textsc{f}}}=K_{\mbox{\textsc{r}}}\circ\left[1_{\varOmega^{\eta}}\varoast K_{\mbox{\textsc{v}}}\right]$
and $K_{\mbox{\textsc{f}}}=K_{\mbox{\textsc{r}}}\circ\left[K_{\mbox{\textsc{v}}}\varoast1_{\varOmega^{\eta}}\right]$
are kernels by composition of kernels, see Section \ref{sub:Composition}.

We are now in the position of defining a kernel that characterizes
the replacement method of a HC algorithm. Before doing that, notice
that $\mbox{\textsc{r}}{}_{\textsc{HC}}\left(x,y\right)=\pi_{1}\left(\mbox{\textsc{s}}_{2}\left(x,y\right)\right)$.\end{proof}
\begin{lem}
\label{lem:HC-Replace-Kernel}The function $\mbox{\textsc{r}}{}_{\textsc{HC}}=\pi_{1}\circ\mbox{\textsc{s}}_{2}$
is measurable and $K_{\mbox{\textsc{r}}{}_{\textsc{HC}}}\equiv1_{\mbox{\textsc{r}}{}_{\textsc{HC}}}$
as defined in theorem \textup{\ref{thm:Kernel_det_f}} is a kernel.\end{lem}
\begin{proof}
Follows from the fact $\mbox{\textsc{r}}{}_{\textsc{HC}}=\pi_{1}\circ\mbox{\textsc{s}}_{2}$
is measurable (composition of measurable functions is measurable)
and theorem \ref{thm:Kernel_det_f}.\end{proof}
\begin{cor}
\label{corHill-Climbing-Kernel}The Hill Climbing algorithm shown
in Algorithm \ref{Alg:Hill-Climbing} can be characterized by a kernel
if its \noun{Variate$_{\mbox{\textsc{hc}}}$} stochastic method can
be characterized by a kernel. \end{cor}
\begin{proof}
Follows from example \ref{exa:NextPopHC}, proposition \ref{prop:VR-Kernel}
and lemma \ref{lem:HC-Replace-Kernel}.\end{proof}
\begin{cor}
\label{corPar-Hill-Climbing-Kernel}The Parallel Hill Climbing algorithm
shown in Algorithm \ref{Alg:Parallel-Hill-Climbing} can be characterized
by a kernel if the \noun{Variate$_{\mbox{\textsc{hc}}}$} stochastic
method of the parallelized HC can be characterized by a kernel. \end{cor}
\begin{proof}
Follows from example \ref{exa:NextPopPHC}, corollary \ref{corHill-Climbing-Kernel}
and proposition \ref{prop:Join-stochastic-kernel}.\end{proof}
\begin{lem}
The \noun{NextPop} stochastic method of the \noun{SSGa} shown in Algorithm
\ref{Alg:SSGA} can be characterized by the composition of two kernels
$\mbox{\textsc{v}}_{\mbox{\textsc{SSGa}}}=\left[\left[\mbox{\textsc{v}}\circ\pi_{\left\{ 1,2\right\} }\right]\varoast1\right]\circ K_{\mathscr{P}}$
and $\mbox{\textsc{r}}_{\mbox{\textsc{SSGa}}}=\left[\mbox{\textsc{b}}_{2,4}\circ\pi_{\left\{ 1,\ldots,4\right\} }\right]\varoast\pi_{\left\{ 5,\ldots,n+2\right\} }$
if lines 3-4 can be characterized by a kernel $\mbox{\textsc{v}}\vcentcolon\varOmega^{2}\times\Sigma^{\otimes2}\rightarrow\left[0,1\right]$.\end{lem}
\begin{proof}
Follows from composition of kernels and proposition \ref{prop:NextPop-SSGa}.
\end{proof}

\subsubsection{\label{sub:Elitist-Stochastic-Methods}Elitist Stochastic Methods}

Some \noun{SGoals} use elitist stochastic methods, i.e., if the best
candidate solution obtained after applying the method is at least
as good as the best candidate solution before applying it, in order
to capture the notion of ``improving'' the solution.
\begin{defn}
(\textbf{elitist method}) A stochastic method $\mbox{\textsc{f}}\vcentcolon\Omega^{\eta}\rightarrow\varOmega^{\upsilon}$
is called elitist if $f\left(\mbox{\textsc{Best}}\left(\mbox{\textsc{f}}\left(P\right)\right)\right)\leq f\left(\mbox{\textsc{Best}}\left(P\right)\right)$.\end{defn}
\begin{example}
\label{exa:NextPopElitist}The \noun{NextPop} methods of the following
algorithms\footnote{Here we just present the examples when such algorithms consider neutral
mutations, but it is also valid when those do not consider neutral
mutations (we just need to reverse the product order).}, are elitist stochastic methods. Here, we will denotate $\mbox{Q}{}_{\textsc{A}}\equiv\mbox{\textsc{NextPop}}{}_{\textsc{A}}\left(P\right)$.
\begin{enumerate}
\item \noun{SSGa}: $\mbox{\textsc{Best}}\left(\mbox{Q}{}_{\textsc{SSGa}}\left(P\right)\right)=\mbox{\textsc{Best}}\left(c_{1}\times c_{2}\times P\right)$,
(see Algorithm \ref{Alg:SSGA}). Then, $f\left(\mbox{\textsc{Best}}\left(c_{1}\times c_{2}\times P\right)\right)\leq f\left(\mbox{\textsc{Best}}\left(P\right)\right)$.
\item \noun{HC:} $\mbox{\textsc{Best}}\left(\mbox{\textsc{Q}}{}_{\textsc{HC}}\left(x\right)\right)=\mbox{\textsc{Best}}\left(\mbox{\textsc{Variate\ensuremath{{}_{\textsc{HC}}}}}\left(x\right)\times x\right)$
(see Algorithm \ref{Alg:Hill-Climbing}). Then, $f\left(\mbox{\textsc{Best}}\left(\mbox{\textsc{Variate}\ensuremath{{}_{\textsc{HC}}}}\left(x\right)\times x\right)\right)\leq f\left(x\right)=f\left(\mbox{\textsc{Best}}\left(x\right)\right)$.
\item PHC: Let $k\in\left[1,n\right]$ the index of the best individual
in population $P$, then $f\left(\mbox{\textsc{Best}}\left(P\right)\right)=f\left(P_{k}\right)$.
Since $\mbox{\textsc{Q}}{}_{\textsc{PHC}}\left(P\right)_{i}={\displaystyle \textsc{Q}{}_{\textsc{HC}}\left(P_{i}\right)}$
for all $i=1,2,\ldots,n$ (see Algorithm \ref{Alg:Parallel-Hill-Climbing}),
it is clear that $f\left(\mbox{\textsc{Q}}{}_{\textsc{PHC}}\left(P\right)_{k}\right)\leq f\left(\mbox{\textsc{Best}}\left(P\right)\right)$
(\noun{$\mbox{\textsc{Q}}{}_{\textsc{HC}}$} is elitist). Then, $f\left(\mbox{\textsc{Best}}\left(\mbox{\textsc{Q}}{}_{\textsc{PHC}}\left(P\right)\right)\right)\leq\mbox{\textsc{Q}}{}_{\textsc{PHC}}\left(P\right)_{i}=f\left(\mbox{\textsc{Best}}\left(P\right)\right)$.
\end{enumerate}
\end{example}
\begin{defn}
(\textbf{elitist kernel}) A kernel $K\vcentcolon\varOmega^{\eta}\times\Sigma^{\otimes\upsilon}\rightarrow\left[0,1\right]$
is called elitist if $K\left(x,A\right)=0$ for each $A\in\Sigma^{\otimes\upsilon}$
such that $d\left(x\right)<d\left(y\right)$ for all $y\in A$.\end{defn}
\begin{prop}
\label{prop:Kernels-Elitist-HC-PHC-SSGa}Kernels $\mbox{\textsc{r}}_{\mbox{\textsc{hc}}}$
and $\mbox{\textsc{r}}_{\mbox{\textsc{SSGa}}}$ are elitist kernels.\end{prop}
\begin{proof}
Let $\left(x,y\right)\in\Sigma^{\otimes2}$ and $A\in\Sigma$ such
that $d\left(z\right)<d\left(x,y\right)$ for all $z\in A$. Now,
$\mbox{\textsc{r}}{}_{\textsc{HC}}\left(x,y\right)=\pi_{1}\circ\mbox{\textsc{s}}_{2}\left(x,y\right)$
(def $\mbox{\textsc{r}}{}_{\textsc{HC}}$), clearly, $d\left(\mbox{\textsc{r}}{}_{\textsc{HC}}\left(x,y\right)\right)\leq d\left(x,y\right)$
(def $d\left(\right))$, therefore $d\left(\mbox{\textsc{r}}{}_{\textsc{HC}}\left(x,y\right)\right)\notin A$
(def $A$). In this way, $\mbox{\textsc{r}}{}_{\textsc{HC}}\left(x,A\right)=0$
(def kernel $\mbox{\textsc{r}}{}_{\textsc{HC}}$ and theorem \ref{thm:Kernel_det_f}).
Therefore, \textbf{$\mbox{\textsc{r}}_{\mbox{\textsc{hc}}}$}is\textbf{
}elitist (def elitist kernel). A similar proof is carried on for $\mbox{\textsc{r}}_{\mbox{\textsc{SSGa}}}$.\end{proof}
\begin{lem}
\label{lem:elitist-kernel}If $K\vcentcolon\varOmega^{\eta}\times\Sigma^{\otimes\upsilon}\rightarrow\left[0,1\right]$
is elitist then
\begin{enumerate}
\item $K\left(x,\left(\varOmega_{\overline{d\left(x\right)}}^{v}\right)^{c}\right)=0$
and $K\left(x,\varOmega_{\overline{d\left(x\right)}}^{v}\right)=1$.
\item Let $x\in\varOmega^{\eta}$, if $d\left(x\right)<\alpha\in\mathbb{R}$
then $K\left(x,\left(\varOmega_{\overline{\alpha}}^{v}\right)^{c}\right)=0$
and $K\left(x,\varOmega_{\overline{\alpha}}^{v}\right)=1$
\end{enumerate}
\end{lem}
\begin{proof}
\textbf{{[}1{]}} Let $y\in\left(\varOmega_{\overline{d\left(x\right)}}^{v}\right)^{c}$
then $\neg\left(d\left(y\right)\leq d\left(x\right)\right)$ (def
complement,$\varOmega_{\overline{d\left(x\right)}}$), i.e., $d\left(x\right)<d\left(y\right)$.
Therefore, $K\left(x,\left(\varOmega_{\overline{d\left(x\right)}}^{v}\right)^{c}\right)=0$
($K$ elitist) and $K\left(x,\varOmega_{\overline{d\left(x\right)}}^{v}\right)=1$
($K_{x,\bullet}$ probability measure). \textbf{{[}2{]}} if $d\left(x\right)<\alpha$
then $\varOmega_{\overline{d\left(x\right)}}\subseteq\varOmega_{\overline{\alpha}}$
(def $\Omega_{\epsilon}$) and $\left(\varOmega_{\overline{\alpha}}\right)^{c}\subseteq\left(\varOmega_{\overline{d\left(x\right)}}\right)^{c}$
(def $^{c}$). Clearly, $K\left(x,\left(\varOmega_{\overline{\alpha}}^{v}\right)^{c}\right)\leq K\left(x,\left(\varOmega_{\overline{d\left(x\right)}}^{v}\right)^{c}\right)=0$
and $K\left(x,\varOmega_{\overline{\alpha}}^{v}\right)=1$ ($K_{x,\bullet}$measure).\end{proof}
\begin{defn}
(\textbf{optimal strictly bounded from zero}) A kernel $K\vcentcolon\varOmega^{\eta}\times\Sigma^{\otimes\upsilon}\rightarrow\left[0,1\right]$
is called optimal strictly bounded from zero iff $K\left(x,\varOmega_{\epsilon}\right)\geq\delta\left(\epsilon\right)>0$
for all $\epsilon>0$.
\end{defn}

\section{Convergence of a\noun{ SGoal}}

We will follow the approach proposed by Günter Rudolph in \cite{Rudolph96convergenceof},
to determine the convergence properties of a \noun{SGoal}. In the
rest of this paper, $\Sigma$ is an optimization $\sigma$-algebra.
First, Rudolph defines a convergence property for a \noun{SGoal} in
terms of the objective function. 
\begin{defn}
(\textbf{\noun{SGoal}}\textbf{ convergence}). Let $P_{t}\in\varOmega^{n}$
be the population maintained by a \noun{SGoal} $\mathscr{A}$ at iteration
$t$. Then $\mathscr{A}$ converges to the global optimum if the random
sequence $\left(D_{t}=d\left(P_{t}\right)\vcentcolon t\geq0\right)$
converges completely to zero. 
\end{defn}
Then, Rudolph proposes a sufficient condition on the kernel when applied
to the set of strict $\epsilon$-optimal states in order to attain
such convergence.
\begin{lem}
\label{lem:(Rudolph-1)}(Lemma 1 in \cite{Rudolph96convergenceof})
If $K\left(x,\varOmega_{\epsilon}\right)\geq\delta>0$ for all $x\in\Omega_{\epsilon}^{c}$
and $K\left(x,\varOmega_{\epsilon}\right)=1$ for all $x\in\varOmega_{\epsilon}$
then, equation \ref{eq:RudolphLemma1} holds for $t\geq1$.
\end{lem}
\begin{equation}
K^{\left(t\right)}\left(x,\varOmega_{\epsilon}\right)\geq1-\left(1-\delta\right)^{t}\label{eq:RudolphLemma1}
\end{equation}

\begin{proof}
In \cite{Rudolph96convergenceof}, Rudolph uses induction on $t$
in order to demostrate lemma \ref{lem:(Rudolph-1)}. For $t=1$ we
have that $K^{\left(t\right)}\left(x,\varOmega_{\epsilon}\right)=K\left(x,\varOmega_{\epsilon}\right)$
(equation \ref{eq:Kernel-Iteration}), so $K{}^{\left(t\right)}\left(x,\varOmega_{\epsilon}\right)\geq\delta$
(condition lemma), therefore $K{}^{\left(t\right)}\left(x,\varOmega_{\epsilon}\right)\geq1-\left(1-\delta\right)^{t}$
($t=1$ and numeric operations). Here, we will use the notations $K{}^{\left(t\right)}\left(y,\varOmega_{\epsilon}\right)=K_{y}{}^{\left(t\right)}\left(\varOmega_{\epsilon}\right)$
to reduce the visual length of the equations.

\noindent %
\begin{tabular}{ll}
$K_{x}^{\left(t+1\right)}\left(\varOmega_{\epsilon}\right)$ & \tabularnewline
$={\displaystyle \intop_{\varOmega}}K_{y}^{\left(t\right)}\left(\varOmega_{\epsilon}\right)K\left(x,dy\right)$ & (equation \ref{eq:Kernel-Iteration})\tabularnewline
$={\displaystyle \intop_{\varOmega_{\epsilon}}}K_{y}^{\left(t\right)}\left(\varOmega_{\epsilon}\right)K\left(x,dy\right)+{\displaystyle \intop_{\varOmega_{\epsilon}^{c}}}K_{y}^{\left(t\right)}\left(\varOmega_{\epsilon}\right)K\left(x,dy\right)$ & ($\varOmega=\varOmega_{\epsilon}\bigcup\varOmega_{\epsilon}^{c}$)\tabularnewline
$={\displaystyle \intop_{\varOmega_{\epsilon}}}K\left(x,dy\right)+{\displaystyle \intop_{\varOmega_{\epsilon}^{c}}}K_{y}^{\left(t\right)}\left(\varOmega_{\epsilon}\right)K\left(x,dy\right)$ & (If $y\in\varOmega_{\epsilon},\,K_{y}^{\left(t\right)}\left(\varOmega_{\epsilon}\right)=1$)\tabularnewline
$=K\left(x,\varOmega_{\epsilon}\right)+{\displaystyle \intop_{\varOmega_{\epsilon}^{c}}}K_{y}^{\left(t\right)}\left(\varOmega_{\epsilon}\right)K\left(x,dy\right)$ & (def kernel)\tabularnewline
$\geq K\left(x,\varOmega_{\epsilon}\right)+\left[1-\left(1-\delta\right)^{t}\right]{\displaystyle \intop_{A_{\epsilon}^{c}}}K\left(x,dy\right)$ & (Induction hypothesis)\tabularnewline
$\geq K\left(x,\varOmega_{\epsilon}\right)+\left[1-\left(1-\delta\right)^{t}\right]K\left(x,\varOmega_{\epsilon}^{c}\right)$ & (del kernel)\tabularnewline
$\geq K\left(x,\varOmega_{\epsilon}\right)+K\left(x,\varOmega_{\epsilon}^{c}\right)-\left(1-\delta\right)^{t}K\left(x,\varOmega_{\epsilon}^{c}\right)$ & \tabularnewline
$\geq1-\left(1-\delta\right)^{t}\left(1-K\left(x,\varOmega_{\epsilon}\right)\right)$ & (Probability)\tabularnewline
$\geq1-\left(1-\delta\right)^{t}\left(1-\delta\right)$ & (condition lemma)\tabularnewline
$\geq1-\left(1-\delta\right)^{t+1}$ & \tabularnewline
\end{tabular}

\end{proof}
Using lemma \ref{lem:(Rudolph-1)}, Rudolph is able to stay a theorem
for convergence of evolutionary algorithms (we rewrite it in terms
of \noun{SGoal}s). However, Rudolph's proof is not wright, since $Pr\left\{ d\left(P_{t}\right)<\epsilon\right\} =Pr\left\{ P_{t}\in\varOmega_{\epsilon}\right\} $
for $t\geq0$ by definition of $\varOmega_{\epsilon}$ and Rudolph
wrongly assumed that $Pr\left\{ d\left(P_{t}\right)\leq\epsilon\right\} =Pr\left\{ P_{t}\in\varOmega_{\epsilon}\right\} $.
Here, we correct the proof proposed by Rudolph (see step 7 in our
demostration).
\begin{thm}
\label{thm:(Rudolph-1)}(Theorem 1 in Rudolph \cite{Rudolph96convergenceof})
A \noun{SGoal}, whose stochastic kernel satisfies the precondition
of lemma \ref{lem:(Rudolph-1)}, will converge to the global optimum
($f^{*}$) of a real valued function $f:\varPhi\rightarrow\mathbb{R}$
with $f>-\infty$, defined in an arbitrary space $\varOmega\subseteq\varPhi$,
regardless of the initial distribution $p\left(\cdot\right)$.\end{thm}
\begin{proof}
The idea is to show that the random sequence ($d\left(P_{t}\right)\vcentcolon t\geq0$)
converges completely to zero under the pre-condition of lemma \ref{lem:(Rudolph-1)}
\cite{Rudolph96convergenceof}. 

\begin{tabular}{rll}
$Pr\left\{ P_{t}\in\varOmega_{\epsilon}\right\} $ & $={\displaystyle \intop_{\varOmega}}K^{\left(t\right)}\left(y,\varOmega_{\epsilon}\right)p\left(dx\right)$ & (Kernel definition)\tabularnewline
 & $\geq1-\left(1-\delta\right)^{t}{\displaystyle \intop_{\varOmega}}p\left(dx\right)$ & (Lemma \ref{lem:(Rudolph-1)})\tabularnewline
 & $\geq1-\left(1-\delta\right)^{t}$ & ($p\left(\cdot\right)$ probability)\tabularnewline
$-Pr\left\{ P_{t}\in\varOmega_{\epsilon}\right\} $ & $\leq\left(1-\delta\right)^{t}-1$ & (Reversing order)\tabularnewline
$1-Pr\left\{ P_{t}\in\varOmega_{\epsilon}\right\} $ & $\leq\left(1-\delta\right)^{t}$ & (Adding $1$ to both sides)\tabularnewline
$Pr\left\{ d\left(P_{t}\right)<\epsilon\right\} $ & $\leq Pr\left\{ d\left(P_{t}\right)\leq\epsilon\right\} $ & (Probability)\tabularnewline
$Pr\left\{ P_{t}\in\varOmega_{\epsilon}\right\} $ & $\leq Pr\left\{ d\left(P_{t}\right)\leq\epsilon\right\} $ & (Definition $\varOmega_{\epsilon}$)\tabularnewline
$-Pr\left\{ d\left(P_{t}\right)\leq\epsilon\right\} $ & $\leq-Pr\left\{ P_{t}\in\varOmega_{\epsilon}\right\} $ & (Organizing)\tabularnewline
$1-Pr\left\{ d\left(P_{t}\right)\leq\epsilon\right\} $ & $\leq1-Pr\left\{ P_{t}\in\varOmega_{\epsilon}\right\} $ & (Adding 1 to both sides)\tabularnewline
$Pr\left\{ d\left(P_{t}\right)>\epsilon\right\} $ & $\leq1-Pr\left\{ P_{t}\in\varOmega_{\epsilon}\right\} $ & (Probability)\tabularnewline
 & $\leq\left(1-\delta\right)^{t}$ & (Transitivity with line 5)\tabularnewline
\end{tabular}

Since $\left(1-\delta\right)^{t}\rightarrow0$ as $t\rightarrow\infty$
then $Pr\left\{ d\left(P_{t}\right)>\epsilon\right\} \rightarrow0$
as $t\rightarrow\infty$, so $D_{t}\overset{p}{\rightarrow}0$. Now,

\begin{tabular}{cll}
${\displaystyle \sum_{i=1}^{\infty}}Pr\left\{ d\left(P_{t}\right)>\epsilon\right\} $ & $\leq{\displaystyle \sum_{i=1}^{\infty}}\left(1-\delta\right)^{t}$ & (line 11)\tabularnewline
 & $\leq\frac{\left(1-\delta\right)}{\delta}$ & (geometric serie)\tabularnewline
 & $<\infty$ & \tabularnewline
\end{tabular}

Therefore, ($d\left(P_{t}\right)\vcentcolon t\geq0$) converges completely
to zero.

\end{proof}

\subsection{Convergence of a\noun{ VR-SGoal}}

We follow the approach proposed by Günter Rudolph in \cite{Rudolph96convergenceof},
to determine the convergence properties of a \noun{VR-SGoal}s but
we formalize it in terms of kernels (both variation and replacement). 
\begin{thm}
\label{thm:(Theorem-2-in-Rudolph)}A \noun{VR-SGoal} with $K_{\mbox{\textsc{v}}}$
an optimal strictly bounded from zero variation kernel and $K_{\mbox{\textsc{r}}}$
an elitist replacement kernel, will converge to the global optimum
of the objective function.\end{thm}
\begin{proof}
If we prove that $K=K_{\mbox{\textsc{r}}}\circ\left[K_{\mbox{\textsc{v}}}\varoast1_{\varOmega^{\eta}}\right]$
satisfies the precondition of lemma \ref{lem:(Rudolph-1)} then the
\noun{VR-SGoal} will converge to the global optimum of the objective
function (theorem \ref{thm:(Rudolph-1)}). we use the notation $\omega=\eta+\upsilon$
in this proof.

\textbf{{[}1. $K\left(x,A\right)={\displaystyle \intop_{\varOmega^{\omega}\times\left\{ x\right\} }}\left[K_{\mbox{\textsc{v}}}\varoast1_{\varOmega^{\eta}}\right]\left(x,dy\right)K_{\mbox{\textsc{r}}}\left(y,A\right)${]}}

\noindent \begin{center}
\begin{tabular}{ll}
$K\left(x,A\right)$ & \tabularnewline
$=\left(K_{\mbox{\textsc{r}}}\circ\left[K_{\mbox{\textsc{v}}}\varoast1_{\varOmega^{\eta}}\right]\right)\left(x,A\right)$ & (def $K$)\tabularnewline
$={\displaystyle \intop_{\varOmega^{\omega}}}\left[K_{\mbox{\textsc{v}}}\varoast1_{\varOmega^{\eta}}\right]\left(x,dy\right)K_{\mbox{\textsc{r}}}\left(y,A\right)$ & (def $\circ$)\tabularnewline
$={\displaystyle \intop_{\varOmega^{\omega}}}K_{\mbox{\textsc{v}}}\left(x,\pi_{\left\{ 1,\ldots,\upsilon\right\} }\left(dy\right)\right)1_{\varOmega^{\eta}}\left(x,\pi_{\left\{ \upsilon+1,\ldots,\omega\right\} }\left(dy\right)\right)K_{\mbox{\textsc{r}}}\left(y,A\right)$ & (def $\varoast$)\tabularnewline
$={\displaystyle \intop_{\varOmega^{\omega}\times\left\{ x\right\} }}K_{\mbox{\textsc{v}}}\left(x,\pi_{\left\{ 1,\ldots,\upsilon\right\} }\left(dy\right)\right)1_{\varOmega^{\eta}}\left(x,\pi_{\left\{ \upsilon+1,\ldots,\omega\right\} }\left(dy\right)\right)K_{\mbox{\textsc{r}}}\left(y,A\right)$ & (def $1_{\varOmega^{\eta}}$)\tabularnewline
$={\displaystyle \intop_{\varOmega^{\omega}\times\left\{ x\right\} }}\left[K_{\mbox{\textsc{v}}}\varoast1_{\varOmega^{\eta}}\right]\left(x,dy\right)K_{\mbox{\textsc{r}}}\left(y,A\right)$ & (def $K$)\tabularnewline
\end{tabular}
\par\end{center}

Notice, if $y\in\varOmega^{\upsilon}\times\left\{ x\right\} $ then
$d\left(y\right)\leq d\left(x\right)$ (def $d\left(\right)$) and
if $y\in\varOmega_{\epsilon}^{\omega}$ then $d\left(y\right)<\epsilon$
(def $\varOmega_{\epsilon}^{\omega}$) therefore $K_{\mbox{\textsc{r}}}\left(y,\varOmega_{\epsilon}^{\eta}\right)=1$
(lemma \ref{lem:elitist-kernel}.2).

\textbf{{[}2. $K\left(x,\varOmega_{\epsilon}^{\eta}\right)\geq\delta\left(\epsilon\right)>0$
for all $x\in\varOmega^{\eta}${]}}

\noindent \begin{center}
\begin{tabular}{lll}
$K\left(x,\varOmega_{\epsilon}^{\eta}\right)$ & $={\displaystyle \intop_{\varOmega_{\epsilon}^{\omega}\bigcup\left(\varOmega_{\epsilon}^{w}\right)^{c}}}\left[K_{\mbox{\textsc{v}}}\varoast1_{\varOmega^{\eta}}\right]\left(x,dy\right)K_{\mbox{\textsc{r}}}\left(y,\varOmega_{\epsilon}^{\eta}\right)$ & (obvious)\tabularnewline
 & $\geq{\displaystyle \intop_{\varOmega_{\epsilon}^{\omega}}}\left[K_{\mbox{\textsc{v}}}\varoast1_{\varOmega^{\eta}}\right]\left(x,dy\right)*K_{\mbox{\textsc{r}}}\left(y,\varOmega_{\epsilon}^{\eta}\right)$ & ($K_{x,\bullet}$ measure)\tabularnewline
 & $\geq{\displaystyle \intop_{\varOmega_{\epsilon}^{\omega}}}\left[K_{\mbox{\textsc{v}}}\varoast1_{\varOmega^{\eta}}\right]\left(x,dy\right)$ & (lemma \ref{lem:elitist-kernel}.2)\tabularnewline
 & $\geq{\displaystyle \intop_{\varOmega_{\epsilon}^{\omega}\times\left\{ x\right\} }}K_{\mbox{\textsc{v}}}\left(x,\pi_{\left\{ 1,\ldots,\upsilon\right\} }\left(dy\right)\right)$ & (line 3)\tabularnewline
 & $\geq\int_{\varOmega_{\epsilon}^{\omega}}K_{\mbox{\textsc{v}}}\left(x,\pi_{\left\{ 1,\ldots,\upsilon\right\} }\left(dy\right)\right)$ & (obvious)\tabularnewline
 & $\geq\int_{\varOmega_{\epsilon}^{\upsilon}}K_{\mbox{\textsc{v}}}\left(x,dz\right)$ & (notation)\tabularnewline
 & $\geq K_{\mbox{\textsc{v}}}\left(x,\varOmega_{\epsilon}^{\omega}\right)$ & (def kernel)\tabularnewline
\end{tabular}
\par\end{center}

Clearly, $K\left(x,\varOmega_{\epsilon}^{\eta}\right)\geq\delta\left(\epsilon\right)>0$
for all $x\notin\varOmega_{\epsilon}^{\eta}$ ($K_{\mbox{\textsc{v}}}$
optimal strictly bounded from zero).

\textbf{{[}3. $K\left(x,\varOmega_{\epsilon}^{\eta}\right)=1$ if
$x\in\varOmega_{\epsilon}^{\eta}${]}} If $x\in\varOmega_{\epsilon}^{\eta}$
then $d\left(x\right)<\epsilon$ (def $\varOmega_{\epsilon}^{\eta}$).
Clearly, $d\left(y\right)<\epsilon$ (transitivity), therefore $K\left(x,\left(\varOmega_{\epsilon}^{\eta}\right)^{c}\right)=0$
(lemma \ref{lem:elitist-kernel}.2) and $K\left(x,\varOmega_{\epsilon}^{\eta}\right)=1$
($K_{x,\bullet}$ probability measure).

\end{proof}
\begin{cor}
Algorithms HC, and \noun{SSGa} will converge to the global optimum
of the objective function if kernels$\mbox{\textsc{v}}_{\mbox{\textsc{hc}}}$,
and $\mbox{\textsc{v}}_{\mbox{\textsc{SSGa}}}$ are optimal strictly
bounded from zero kernels.\end{cor}
\begin{proof}
Follows from theorem \ref{thm:(Theorem-2-in-Rudolph)} and proposition
\ref{prop:Kernels-Elitist-HC-PHC-SSGa}.
\end{proof}

\section{Conclusions and Future Work}

Developing a comprehensive and formal approach to stochastic global
optimization algoritms (\noun{SGoals}) is not an easy task due to
the large number of different \noun{SGoals} reported in the literature
(we just formalize and characterize three classic \noun{SGoals} in
this paper!). However, such \noun{SGoals} are defined as joins, compositions
and/or random scans of some common deterministic and stochastic methods
that can be represented as kernels on an appropiated structure (measurable
spaces with some special property and provided with additional structure).
Such special structure is the optimization space (defined in this
paper). On this structure, we are able to characterize several \noun{SGoals}
as special cases of variation/replacement strategies, join strategies,
elitist strategies and we are able to inherit some properties of their
associated kernels. Moreover, we are able to prove convergence properties
(following Rudolph approach \cite{Rudolph96convergenceof}) of \noun{SGoal}s.
Since the optimization $\sigma$-algebra property of the structure
is preserved by product $\sigma$-algebras, our formal approach can
be applicable to both single point \noun{SGoals} and population based
\noun{SGoals}.

Although the theory developed in this paper is comprehensive for just
studying \noun{SGoals} with fixed parameters (like population size
and variation rates), it is a good starting point for studying adapting
\noun{SGoal}s (\noun{SGoal}s that adapt/vary some search parameters
as they are iterating). The central concept for doing that will be
the join of kernels (if we consider the space of the parameter values
as part of the $\sigma$-algebra). However, such study is far from
the scope of this paper. 

Our future work will concentrate on including in this formalization,
as many as possible, selection mechanisms that are used in \noun{SGoal}s,
and extending and developing the theory required for characterizing
both adaptable and Mixing \noun{SGoal}s.

\bibliographystyle{ieeetr}

\end{document}